%% file: main.tex
\documentclass{article} % For LaTeX2e

\input{math_commands.tex}

\input{packages}

\title{
\raisebox{-0.25ex}{\includegraphics[height=\ht\strutbox]{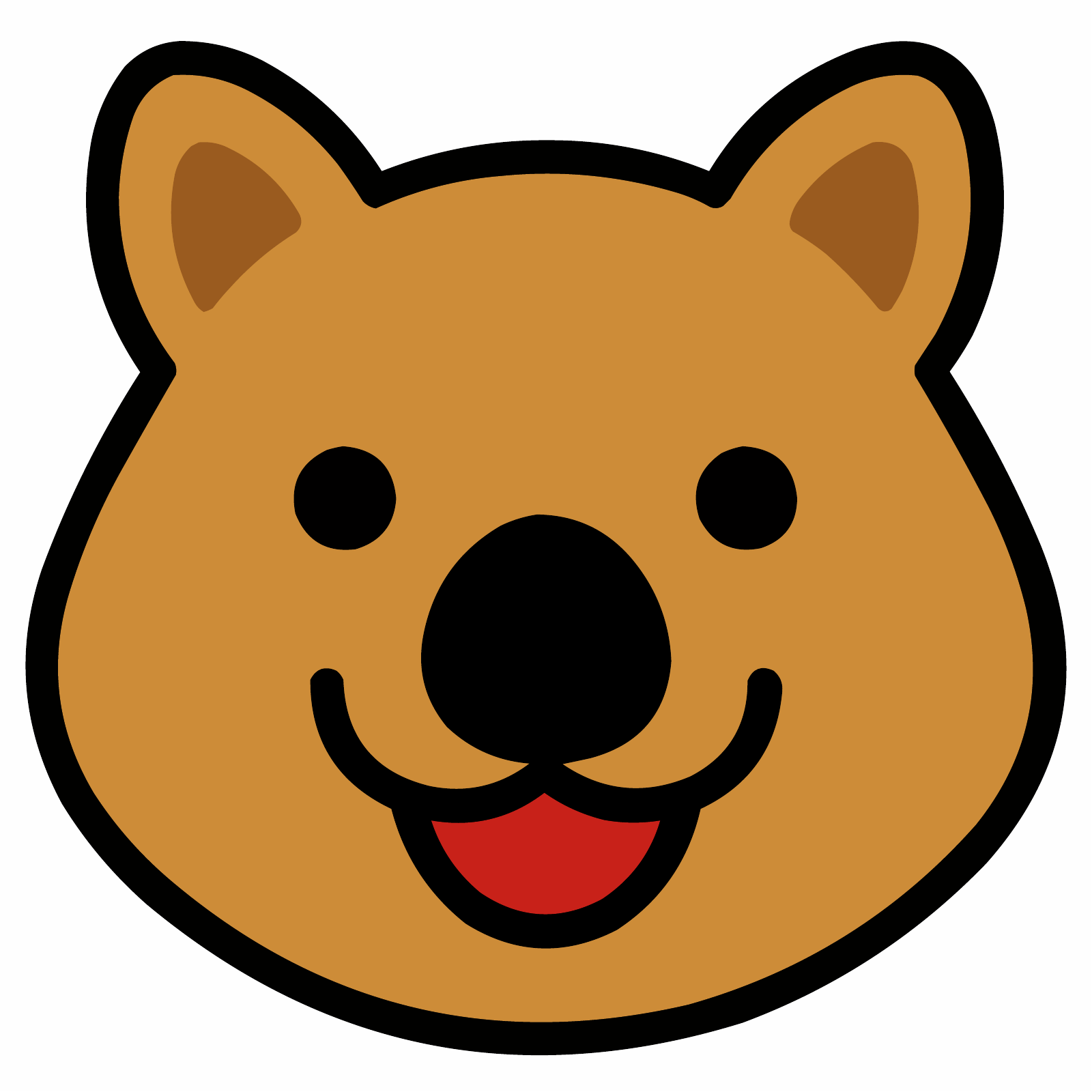}}
\ourmethod{}: Query-oriented KV selection for Efficient LLM Prefill
}
\author{Dalton Jones, Junyoung Park,  Matthew Morse, Mingu Lee, Chris Lott, M. Harper Langston \\
\hspace{4.7cm} Qualcomm AI Research \thanks{Qualcomm AI Research is an initiative of Qualcomm Technologies, Inc.}  \\
\texttt{\{daltjone, junpark, mingul, clott, hlangsto\}@qti.qualcomm.com} \\
}

\iclrfinalcopy % Uncomment for camera-ready version, but NOT for submission.
\begin{document}

\maketitle

\begin{abstract}
% We present \ourmethod{}, a training-free and hardware-agnostic sparse attention mechanism designed to accelerate transformer inference under chunked prefill. 
We present \ourmethodfull{}, a training-free and hardware agnostic sparse attention algorithm for accelerating transformer inference under chunked prefill.
% The key insight is that attention to each key is dominated by a small subset of queries, which can be identified through their angular similarity to the mean query. By leveraging this query--key geometry, 
While many queries focus on a smaller group of keys in the attention operator, we observe that queries with low cosine similarity with respect to the mean query interact more strongly with more keys and have the greatest contribution to final attention logits.
By prioritizing these low cosine similarity queries, the behavior of full attention during the prefill stage can be closely approximated.
\ourmethod{} leverages this observation, accelerating attention by (1) first retaining a small set of representative queries and (2) then subselecting the keys most aligned with those queries.
% efficiently subselects the most relevant keys and values without requiring retraining or custom kernels. 
% Experiments show that \ourmethod{} achieves near-dense accuracy on long-context benchmarks such as LongBench and RULER with only a fraction of the KV budget, outperforms the full model on Needle-In-A-Haystack, and demonstrates strong performance on the Math500 reasoning benchmark, where it even surpasses dense attention in some cases. In addition, \ourmethod{} provides substantial efficiency gains, including up to $5\times$ speedup for the attention module and $3\times$ reduction in time-to-first-token for end-to-end inference.
Through experiments on Needle-In-A-Haystack, LongBench, RULER, and Math500, we show that, while realizing a $3\times$ reduction in time-to-first-token, $5\times$ speedup in attention on Nvidia GPUs and up to nearly a $7\times$ speedup on Intel Xeon CPUs, \ourmethod{} achieves near-baseline accuracy, utilizing 88\% fewer key-value pairs per attention evaluation.
\end{abstract}

\input{contents/intro_condensed}

\input{contents/background_trimmed}
\input{contents/method_trimmed}

\input{contents/results}

\input{contents/related_work}

\input{contents/conclusion}

\input{contents/ethics_and_reproducibility}

\bibliography{reference}
\bibliographystyle{iclr2026_conference}

% \appendix
% \section{Appendix}
\input{contents/appendix}

\end{document}

%% file: math_commands.tex
\usepackage{amsmath,amsfonts,bm}

\def\eqref#1{equation~\ref{#1}}
\def\1{\bm{1}}

\DeclareMathAlphabet{\mathsfit}{\encodingdefault}{\sfdefault}{m}{sl}
\SetMathAlphabet{\mathsfit}{bold}{\encodingdefault}{\sfdefault}{bx}{n}

\DeclareMathOperator*{\minimize}{minimize}

%% file: packages.tex
\usepackage{iclr2026_conference,times}

\usepackage{graphicx} % Required for inserting images
\usepackage{url}
\usepackage[dvipsnames]{xcolor}
\usepackage{multirow}

\usepackage{graphicx} % Required for inserting images
\usepackage{microtype}
\usepackage{graphicx}
\usepackage{booktabs} % for professional tables
\usepackage[dvipsnames,svgnames]{xcolor}
\usepackage{caption}
\usepackage{subcaption}

\usepackage[colorlinks=true, linkcolor=NavyBlue, citecolor=NavyBlue, urlcolor=magenta]{hyperref}
\usepackage[capitalize,noabbrev,nameinlink]{cleveref}

\usepackage{color}

\newcommand{\ourmethod}{\textsc{QuoKA}}
\newcommand{\ourmethodfull}{\textsc{QuoKA}: 
\underline{Qu}ery-\underline{o}riented \underline{K}V selection for efficient \underline{A}ttention}

\usepackage{enumitem} 

\usepackage{algorithm}
\usepackage{algpseudocode}
\usepackage{amsthm}

\usepackage{adjustbox}

\usepackage{color}

\usepackage{wrapfig}

\newtheorem{theorem}{Theorem}

\usepackage[normalem]{ulem}

\usepackage{array}

\usepackage{xfrac}

\usepackage{xargs}

\usepackage[T1]{fontenc}

\newcommandx{\JP}[2][1=]{%
  \todo[inline,
        linecolor=blue,
        color=blue!25,
        bordercolor=blue,
        #1,
        size=\small]{JP: #2}%
}

\newcommandx{\DJJ}[2][1=]{%
  \todo[inline,
        linecolor=red,
        backgroundcolor=red!25,
        bordercolor=red,
        #1,
        size=\small]{DJ: #2}%
}
\newcommandx{\MM}[2][1=]{%
  \todo[inline,
        linecolor=orange,
        backgroundcolor=orange!25,
        bordercolor=orange,
        #1,
        size=\small]{MM: #2}%
}

%% file: contents/intro_condensed.tex
\section{Introduction}
\label{sec:intro}

A major bottleneck in large language model (LLM) inference is prefill latency, which can account for more than 70\% of total runtime. This latency is especially significant on CPUs, consumer-grade GPUs, and edge accelerators, where resources are limited \citep{agrawal2024taming,kamath2025pod, xu2025fast}. To mitigate this, recent deployments increasingly adopt \textit{chunked prefill}, which divides input into blocks to improve scheduling and utilization \citep{agrawal2023sarathi,lai2025flexprefill,kwon2023efficient}. Nevertheless, due to quadratic complexity of the underlying attention, chunked prefill remains computationally expensive.
% A major bottleneck in large language model (LLM) inference is prefill latency: processing long prompts before generating the first token can account for 60-70\% of total runtime, especially on CPUs, consumer GPUs, and edge devices with limited resources \needref{}. To mitigate this cost and satisfy hardware constraints, on-device deployment of LLMs often employs \textit{chunked prefill}, which divides input into blocks to improve scheduling and utilization \citep{agrawal2023sarathi,lai2025flexprefill,kwon2023efficient}. Nevertheless, chunked prefill remains expensive because the underlying attention operation has quadratic complexity. 
% Recent methods such as sparse attention seek to mitigate this by identifying and exploiting sparsity in attention matrices to decrease complexity.
Methods such as sparse attention seek to overcome this complexity by identifying and exploiting sparsity in attention. Combining chunked prefill with sparse attention offers a promising path to sub-quadratic complexity and substantial latency improvements in resource constrained environments.

% Sparse attention methods can be grouped into two families. Kernel-level approaches sparsify QK multiplication with block or fixed patterns \citep{child2019generating,dao2022flashattention,xu2025xattention,jiang2024minference,gao2024seerattention,zhang2025spargeattn}. Although they reduce FLOPs and have been extensively studied in optimized GPU kernels, they are less effective for chunked prefill due to design mismatch; small per-chunk query blocks reduce arithmetic intensity, rigid patterns incur overhead, and custom kernels hinder portability. This makes kernel-level sparsity difficult to translate into reliable chunked prefill speedups in practice. In contrast, KV-level methods operate directly on the KV cache, subselecting relevant key--value pairs to approximate the full attention output \citep{ribar2024sparq,tang2024quest,zhu2024sampleattention,yang2025moretrainingfreesparseattention,singhania2024loki}. This directly reduces cache size and memory traffic while remaining compatible with dense kernels, making them portable across heterogeneous hardware. However, since existing methods mainly target generation, where KV pairs are chosen for a single query, they often lose accuracy in prefill that requires aggregating relevance across many queries.

Sparse attention algorithms can be broadly categorized into two families: \textit{pattern-based} and \textit{query-dependent} approaches. Pattern-based approaches impose fixed sparsity patterns (e.g., block, strided, or banded) on $QK^\top$ \citep{child2019generating,dao2022flashattention,xu2025xattention,jiang2024minference,gao2024seerattention,zhang2025spargeattn}, often achieving speedups through kernel-level optimizations; however, due to dynamic compute graph and KV cache memory bandwidth overhead under chunked prefill, their benefits are limited.  Furthermore, reliance on custom kernels limits the portability of pattern-based approaches across heterogeneous hardware. 
%These methods often achieve speedups through kernel-level optimizations, but suffer on long context chunked prefill tasks due to KV cache memory bandwidth requirements.

In contrast, query-dependent approaches operate directly on the KV cache, adaptively subselecting the most relevant KVs for a given query \citep{ribar2024sparq,tang2024quest,zhu2024sampleattention,yang2025moretrainingfreesparseattention,singhania2024loki}.  While remaining compatible with optimized kernels and offering strong portability benefits, this strategy reduces both attention complexity and memory traffic. However, existing query-dependent methods are primarily designed for generation, where KVs are selected for a single query; in such cases, it is more straightforward to determine which KVs are relevant for the given query. During prefill, when relevant KVs are selected for many queries at once, this can result in significant performance degradations. Under chunked prefill, where important KVs are repeatedly subselected for multiple queries, these degradations become more pronounced. 

% To address this gap, we present \ourmethodfull{}, a geometry-guided KV-level sparse attention method optimized for chunked prefill. The key insight is that attention to each key is dominated by a small subset of queries identifiable by their angular similarity, enabling aggressive KV cache reduction while preserving accuracy.
To address the drawbacks of existing approaches, we propose \ourmethodfull{}, a \textit{training-free sparse attention method optimized for chunked prefill} built upon the following observation: \textbf{queries with lower cosine similarity to the mean query attend to the majority of keys}.

\ourmethod{} leverages this observation to retain a small subset of representative queries and subselect KVs with which they strongly interact.  While preserving accuracy, \ourmethod{} accelerates attention during prefill and reduces the number of KVs through three steps:  \emph{Query Subselection}, which retains only the most informative queries; \emph{Cosine-Similarity Scoring}, which provides a stable, bounded proxy for query-key relevance; and \emph{Group-Aware Aggregation}, which efficiently preserves compatibility with modern architectures such as grouped-query attention (GQA) \citep{ainslie2023gqa}. These steps can be seen in~\cref{fig:chunked_prefill_with_ourmethod}, discussed in greater detail in Section~\ref{sec:method}. Unlike kernel-level sparsity, which depends on custom primitives, \ourmethod{} relies on standard linear algebra operations, allowing for platform portability and straightforward deployment. Our contributions are summarized as follows: 
\begin{itemize}[leftmargin=*, itemsep=0.3em, parsep=0pt]
\item \textbf{\ourmethod{}}: a hardware-agnostic and training-free query-oriented sparse attention method for chunked prefill, built on standard linear algebra kernels.
\item \textbf{Accuracy under sparsity:} near-baseline results on long-context benchmarks (Needle-in-a-Haystack, RULER, LongBench, Math500), outperforming existing sparse attention methods.
\item \textbf{Latency reduction:} up to $5\times$ attention speedup and $3\times$ lower time-to-first-token (TTFT) on enterprise-class GPUs, and $7\times$  and 5-6$\times$ speedups on CPUs and on consumer GPUs.
\item \textbf{Generalization across architectures:} validated on diverse decoder-only LLM families (Llama3, Qwen3, SmolLM, GPT-OSS) and on RoPE/NoPE and MoE-based LLMs.
\item \textbf{Robustness to hyperparameters:} accuracy degrades gradually with sparsity and remains stable across parameter choices, enabling deployment under varied constraints.
\end{itemize}

\begin{figure}[t]
\vspace{-3mm}
     \centering
     \begin{subfigure}[b]{0.63\linewidth}
         \centering
         \includegraphics[width=\textwidth]{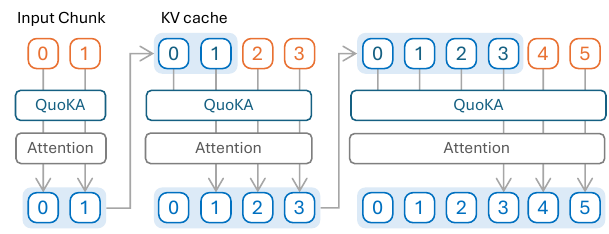}
         \caption{Chunked Prefill with \ourmethod{}}         \label{subfig:chunked_prefill_with_ourmethod_a}
     \end{subfigure}
     \hfill
     \begin{subfigure}[b]{0.33\linewidth}
         \centering
         \includegraphics[width=\textwidth]{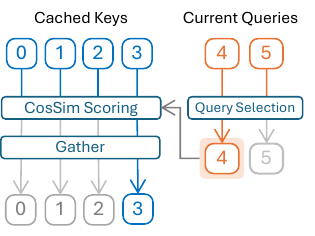}
         \caption{KV selection with \ourmethod{}}
         \label{subfig:chunked_prefill_with_ourmethod_b}
     \end{subfigure}
     \caption{\textbf{Overview of chunked prefill with \ourmethod{}:} \textbf{(a)} A prompt of 6 tokens is divided into three chunks of 2 tokens each. For each chunk, \ourmethod{} subselects the KV cache and feeds the reduced cache into a dense attention kernel. \textbf{(b)} Subselection is performed by applying query subselection based on cosine dissimilarity, followed by key subselection using query–key cosine similarity.}
     \label{fig:chunked_prefill_with_ourmethod}
\vspace{-3mm}
\end{figure}

%% file: contents/background_trimmed.tex
\section{Background}
\label{sec:background}
% \vspace{-1mm}
\subsection{Transformers and Attention}
% \vspace{-1mm}
Since their introduction by \cite{vaswani2017attention}, transformers have become the dominant sequence modeling paradigm. A transformer block consists of two main computational steps: the attention operator and a feed-forward network (FFN). In this work, we focus on the attention primarily. 

For an input sequence $X = [x_1, \ldots, x_T]$, the attention operator $\mathrm{Attn}(X)$ is defined as
\begin{equation}
\label{eq:attention}
% \mathrm{Attn}(X) = \mathrm{Softmax}\left( \frac{QK^\top}{\sqrt{d}} + M \right)V = AV,
\mathrm{Attn}(X) = \mathrm{Softmax}\left(QK^\top / \sqrt{d} + M \right)V = AV,
\end{equation}
where keys $K = XW_K$, queries $Q = XW_Q$, and values $V = XW_V$ are learned projections of the input tokens, with $W_K, W_Q, W_V \in \mathbb{R}^{d \times d}$. With zeros below the diagonal and $-\infty$ elsewhere, the mask $M$ enforces autoregressive causality.

\subsection{Attention latency: prefill vs generation}
Autoregressive text generation in transformer LLMs can be divided into two stages: \textit{prefill} and \textit{generation}, both with distinct latency characteristics. During the \textit{prefill} stage, the entire input prompt is processed to initialize the KV cache, requiring $O(T^2)$ operations from dot products between $T$ queries and $T$ keys in addition to the FFN computation for each token. As a result for long prompts, attention dominates prefill latency. Optimized kernels such as FlashAttention \citep{dao2022flashattention,dao2023flashattention,shah2024flashattention} improve memory locality and utilization, but they do not reduce the quadratic complexity; asymptotic speedups can be achieved by reducing the effective number of KVs attended to by each query. In the generation stage, a single new query attends to the $T$ cached  KVs. For short generation lengths, performance is memory-bound by the large transfer cost of FFN weights. However for longer outputs, the KV cache dominates the memory footprint and bandwidth. Reducing the number of stored KVs both speeds up the attention computation and reduces data transfer, making KV reduction essential for accelerating generation-heavy tasks (e.g., reasoning or code generation).

\input{algorithms/ourmethod}
\subsection{Chunked Prefill}
% \vspace{-1mm}
Given the causality in \cref{eq:attention}, \textit{chunked prefill} \citep{agrawal2023sarathi,holmes2024deepspeed} partitions
the input sequence $X$ into non-overlapping chunks of size $B_\text{CP}$ and process the chunks sequentially: $X = [X_0, X_1, \ldots, X_{N_{B}-1}], N_B=\lceil T/B_\text{CP} \rceil$, where $X_i = [x_{B_\text{CP}i}, \ldots, x_{B_\text{CP}(i+1)-1}]$ denotes the $i$-th chunk. For chunk $i$, let $K_{<i}$ and $V_{<i}$ denote the concatenation of keys and values from all preceding chunks. Then:
\begin{equation}
\label{eq:block_attention}
% \mathrm{Attn}(X_i) \;=\; \mathrm{Softmax}\!\left(
% \frac{Q_i [\,K_i \;|\; K_{<i}\,]^\top}{\sqrt{d}} + M_i \right)
% [\,V_i \;|\; V_{<i}\,],
\mathrm{Attn}(X_i) \;=\; \mathrm{Softmax}\!\left(
Q_i [\,K_i \;|\; K_{<i}\,]^\top / \sqrt{d} + M_i \right)
[\,V_i \;|\; V_{<i}\,],
\end{equation}
where $[\,\cdot \;|\; \cdot\,]$ denotes concatenation along the sequence dimension and $M_i$ enforces causality within $X_i$ and against future tokens. Chunked prefill is particularly important for inference on edge devices such as mobile hardware or consumer GPUs where bandwidth and memory capacity are constrained. It also improves cloud inference throughput by enabling interleaved prefill and decode requests, increasing GPU utilization \citep{holmes2024deepspeed,agrawal2025efficient}. %Notably, the token generation phase can be interpreted as chunked prefill with $b=1$ and a non-empty KV cache.

% \vspace{-1mm}
\subsection{Attention sparsity and sparse attention}
\label{subsec:attn_sparsity}
% \vspace{-1mm}
High sparsity in the attention matrix $A$ has been widely observed in LLMs \citep{zhang2024h2o,xiaoefficient,oren2024transformers,zhu2024sampleattention,park2025keydiff}. 
% This sparsity can be exploited to reduce the cost of attention by selecting only the most important KVs to incoming query tokens $Q$ into an active KV cache $(\hat{K}, \hat{V})$:
This sparsity allows for a reduction in attention cost by selecting the most relevant KVs for incoming queries into an active cache $(\hat{K}, \hat{V})$ such that:
%. Many sparse attention methods take the following approach:
\begin{align}
\label{eq:sparse_attention}
I = \texttt{topk}(f(Q,K), B_{SA}), \quad 
\hat{K} &= \texttt{gather}(K, I), \quad
\hat{V} = \texttt{gather}(V, I), 
% \\
% \mathrm{SparseAttn}(Q,K,V) &= \mathrm{Softmax}\!\left( \frac{Q\hat{K}^\top}{\sqrt{d}} + M \right)\hat{V}, 
% \mathrm{SparseAttn}(Q,K,V) &= \mathrm{Softmax}\!\left( Q\hat{K}^\top / \sqrt{d} + M \right)\hat{V}, 
\end{align}
where $f$ satisfies
% \vspace{-2mm}
\begin{align}
    \label{eq:min_attention}
    % \minimize_{f(Q,K)} \left|\left| \mathrm{Softmax}\!\left( \frac{QK^\top}{\sqrt{d}} + M \right)V -\mathrm{Softmax}\!\left( \frac{Q\hat{K}^\top}{\sqrt{d}} + M \right)\hat{V}\right|\right|,
    \minimize_{f(Q,K)} \left|\left| \mathrm{Softmax}\!\left( QK^\top / \sqrt{d} + M \right)V -\mathrm{Softmax}\!\left( Q\hat{K}^\top / \sqrt{d} + M \right)\hat{V}\right|\right|,
\end{align}
% \vspace{-1mm}
where $f$ is usually constrained to be more efficient than the attention computation itself.

% Approaches such as \citep{zhao2019explicit,ribar2024sparq,tang2024quest,singhania2024loki,yang2024tidaldecode} have primarily focused on the generation phase, where the scoring function $f$ operates on a single query $Q$. \textcolor{red}{A straightforward extension to the multi-query setting computes attention scores for each query separately and then aggregates them by taking a simple mean across queries. However, this naive averaging dilutes the distinctive contributions of individual queries and often leads to downstream performance degradation (see \cref{table:longbench_compare}).} More recently, studies such as \cite{zhang2025spargeattn,gao2024seerattention,zhu2024sampleattention,jiang2024minference,lai2025flexprefill} have explored handling multiple queries simultaneously; yet many of these approaches rely on custom kernels, which limits their optimized deployment across diverse platforms beyond well-studied NVIDIA GPUs.

Prior work \citep{zhao2019explicit,ribar2024sparq,tang2024quest,singhania2024loki,yang2024tidaldecode} generally focused on the generation phase, where the scoring function $f$ operates on a single query $Q$. Extending these methods to the multiple-query setting with by averaging over queries significantly degrades performance (see \cref{table:longbench_compare}), highlighting the need for better ways to aggregate information from multiple queries. 
While recent work \cite{zhang2025spargeattn,gao2024seerattention,zhu2024sampleattention,jiang2024minference,lai2025flexprefill} attempt to address this, these methods typically depend on custom CUDA kernels on NVIDIA GPUs.
This limits compatibility with hardware-tuned kernels like FlashAttention and hinders broader deployment outside of data centers.  %, including on consumer-grade hardware typical of edge devices.

%% file: algorithms/ourmethod.tex
\begin{algorithm}[t]
\caption{KV cache sub-selection using \ourmethod{}}
\label{alg:ourmethod}
\begin{algorithmic}[1]   % [1] = line numbers
\Require queries $Q$, keys $K$, values $V$, prefill chunk size $B_\text{CP}$, selective attention budget $B_{\text{SA}}$, 
         max queries $N_Q$, number of KV heads $n_{\mathrm{KV}}$, 
         number of attention heads $n_{\mathrm{Q}}$

\Comment{Query Sub-selection}
\If{$B_\text{CP} > N_Q$}
  \State $M_Q \gets \texttt{mean}(Q, \texttt{dim=2})$ % O(|Q|)
  \State $S_Q \gets \mathrm{CosSim}(Q, M_Q)$ % O(|Q|)
  \State $Q \gets \texttt{gather}(\texttt{topk}(-S_Q, N_Q), Q)$ %O(|Q|) bc quickselect
% \Else
  % \State $Q \gets Q$
\EndIf

\Comment{Efficient Cosine Similarity via Pre-aggregation}
\State $Q \gets Q / \texttt{norm}(Q, \texttt{dim=-1})$ \Comment{$(b, n_{q}, N_Q, d)$}
\State $K \gets K / \texttt{norm}(K, \texttt{dim=-1})$ % O(T)
\Comment{$(b, n_{\text{KV}}, T, d)$}

\State $\bar{Q} \gets \texttt{mean}\!\left(Q.\texttt{reshape}(b, n_{\mathrm{KV}}, \tfrac{n_{\mathrm{Q}}}{n_{\mathrm{KV}}}, N_Q, d), \texttt{dim=2}\right)$ \Comment{$(b, n_{KV}, N_Q, d)$} % say reshape -> O(N_Q) and mean is O(N_Q *n_Q/n_{KV}), ignore reshape bv big-O
\State $S \gets \bar{Q} K^\top$ \Comment{$(b, n_{\mathrm{KV}}, N_Q, T)$} 
% O(N_Q*T*d*n_{KV}) (number of dot products * d work per dot prod)
\State $\hat{S} \gets \max(S, \texttt{dim=2})$ \Comment{$(b, n_{\mathrm{KV}}, T)$} % O(N_Q*T)
\State $I \gets \texttt{topk}(\hat{S}, B_{\text{SA}})$ %O(T)
% \State $K^\star \gets \texttt{gather}(K, I)$
% \State $V^\star \gets \texttt{gather}(V, I)$
\State $K^\star, V^\star \gets \texttt{gather}(K, I), \texttt{gather}(V, I)$
\end{algorithmic}
\end{algorithm}

%% file: contents/method_trimmed.tex
\begin{figure}
     \centering
     \begin{subfigure}[b]{0.33\textwidth}
         \centering
         \includegraphics[width=\textwidth]{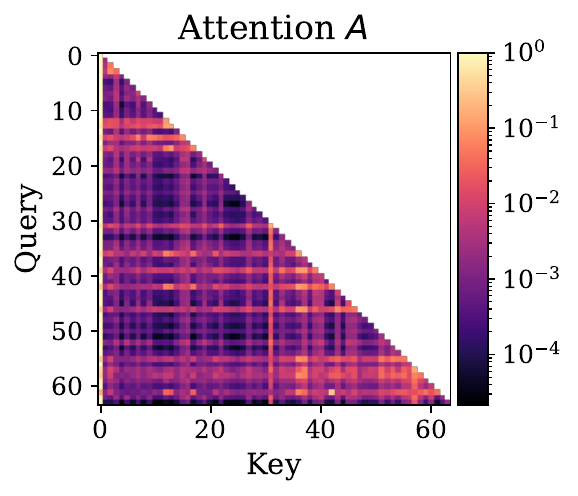}
         \caption{Attention score $A$}
         \label{fig:attention}
     \end{subfigure}
     \hfill     
     \begin{subfigure}[b]{0.30\textwidth}
         \centering
         \includegraphics[width=\textwidth]{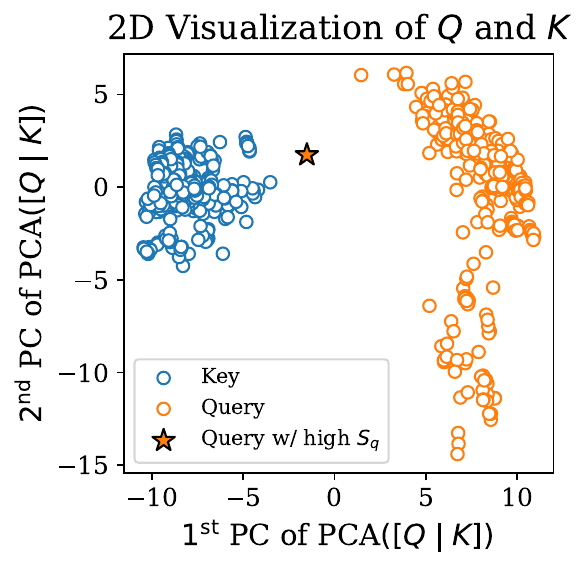}
         \caption{PCA viz. of $Q$ and $K$}
         \label{fig:pca_qk}
     \end{subfigure}
     % \hfill
     \begin{subfigure}[b]{0.30\textwidth}
         \centering
         \includegraphics[width=\textwidth]{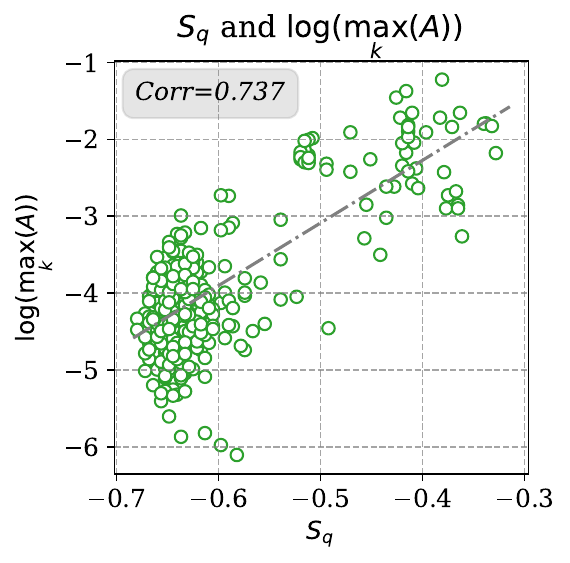}
         \caption{Scatter plot $S_q$ v.s. $\max_k(A)$}
         \label{fig:sq_vs_attn}
     \end{subfigure}     
        \caption{\textbf{Empirical observations from \textit{Llama 3.2-3B-Instruct}, layer 0 head 11.} (a) Attention map $A$. (b) PCA visualization of $Q$ and $K$, showing that queries with higher $S_q$ lie closer to the keys. (c) Correlation between $S_q$ and $\max_k(A)$, indicating stronger key interactions for higher $S_q$ queries.}
        \label{fig:qk_observation}
\vspace{-5mm}
\end{figure}

\section{\ourmethod{} Method}
\label{sec:method}
As discussed in \cref{sec:background}, existing sparse attention methods face limitations in prefill efficiency and portability. \ourmethod{} addresses this by reducing the active KV cache during the chunked prefill process. Specifically, the input sequence $X$ is divided into chunks $\{X_0, X_1, \ldots, X_{N_B-1}\}$, and for each chunk $X_i$, \ourmethod{} selects only the most relevant KVs before computing attention as shown in \cref{fig:chunked_prefill_with_ourmethod}. This is achieved in three stages, as detailed in \cref{alg:ourmethod}. 
\begin{enumerate}[leftmargin=*]
    \item \textbf{Query subselection}: retain only the most informative queries to reduce redundancy.
    \item \textbf{Cosine-Similarity Scoring}: compute cosine similarity between queries and keys to estimate their relevance.
    \item \textbf{Score Aggregation}: combine scores across queries and key-value groups to select the final KV.
\end{enumerate}

The reduced KV set for $X_i$ is then fed into a dense attention kernel such as FlashAttention. By applying this repeatedly, \ourmethod{} reduces prefill cost from $O(T^2)$ to a sub-quadratic complexity.

\subsection{Query Subselection}
\label{subsec:query_subselection}

Prior work \citep{park2025keydiff} examined the distinctive geometric characteristic of keys in LLM. Inspired by this insight, we focus on the geometry of queries. We observe that queries with lower cosine similarity to the mean query tend to align broadly with most keys, while near-mean queries concentrate on a small shared group of keys (\cref{fig:attention}). %For example, the queries indexed between $55$–$60$ consistently show high activation across many keys, whereas those between $45$–$50$ assign high scores to a smaller subset of keys, especially the sink token. 
Thus, utilizing all queries yields redundant information and increases complexity.

To mitigate this, \ourmethod{} retains only those queries that consistently exert significant influence on keys. These queries can be identified using their angular distance from the average query vector $M_Q$. Formally, we rank each query $q$ by $-\mathrm{CosSim}(M_Q, q)$ and retain the top $N_Q$. This preserves the queries most responsible for large key-query dot products while normalizing per query, allowing us to efficiently approximate the post-softmax attention matrix $A$. This can be formalized through the following theorem:
\begin{theorem}\label{theorem:keydiff-key-query}
Consider tokens a fixed query $q_0$ and key $k$, and let the average of a set of queries be denoted $M_Q$. Suppose $\textrm{CosSim}({k}, q_0) = \beta_q > 0$ and $\textrm{CosSim}(M_Q, k) = \alpha_q < 0$. Then 
\begin{equation}
\textrm{CosSim}(M_Q, {q^*}) \leq 1 + \alpha_q \beta_q - 0.5\alpha_q^2 -0.5\beta_q^2.
\label{eq:key-query-eq}
\end{equation}
\end{theorem}

% The proof follows directly from Theorem~3.2 in \citet{park2025keydiff}.
The proof is provided in \cref{sec:proof}. Intuitively, if a query $q$ attends strongly to $k$, then $\beta_q$ will be large while $\alpha_q$ will be small, resulting in a large subselection score 
\[
S_q = -\textrm{CosSim}(M_Q, q^*).
\]
Hence our subselection retains queries that contribute most to the attention distribution, in line with both empirical observations and the underlying geometry of keys and queries in modern LLMs. To further validate our design choice, we analyze the geometry of queries and keys. As shown in \cref{fig:pca_qk}, most queries are separated from the key cluster in a 2D PCA projection, while queries with higher $S_q$ tend to lie closer to the keys. Complementing this, \cref{fig:sq_vs_attn} demonstrates that higher $S_q$ is correlated with larger $\max_k(A)$ outside of the sink token, indicating that such queries exert stronger influence on individual keys. Together, these observations suggest that $S_q$ identifies queries that are both geometrically aligned with keys and dominant in attention.

\subsection{Scoring via Cosine Similarity}
\label{subsec:scoring}
Given the reduced set of queries, the next step is to evaluate their interactions with keys. Existing
methods often use dot products $QK^{\top}$ but these are scale-dependent and unstable under aggregation.
Instead, \ourmethod{} computes
\[
S = \mathrm{CosSim}(Q, K),
\]
which normalizes vectors to unit length and provides a bounded, geometry-aware proxy for softmax
attention weights. Recent work \citep{mongaras2025cottention,park2025keydiff} also supports the use of cosine similarity as a lightweight normalization that approximates softmax behavior. Empirically, ablations on the RULER benchmark in \cref{table:ruler_ablation_cossim} show that cosine similarity improves subselection quality by more than 10\% compared to the dot product. 

\subsection{Aggregation Across Queries and KV Groups}
\label{subsec:score_aggregation}

\begin{wrapfigure}{r}{0.35\textwidth}
\centering
\vspace{-14mm} 
\includegraphics[width=\linewidth]{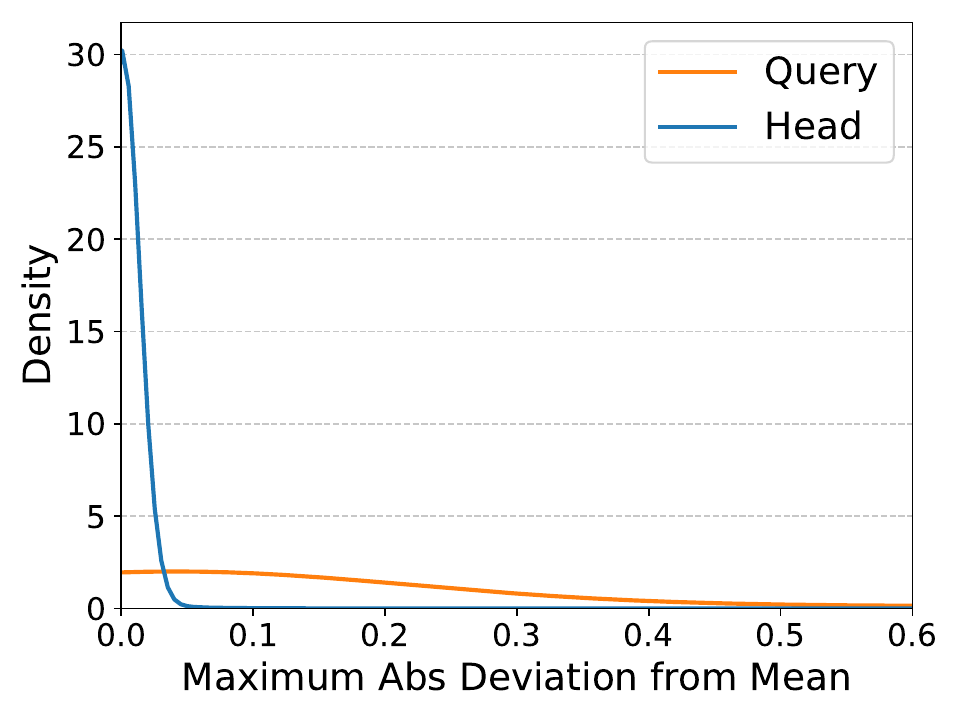}
\caption{Distribution of attention score max deviation from mean along query and head dimension.}
\label{fig:dist_query_head}
\end{wrapfigure}

Aggregation of scores is required along two axes: across queries and across grouped-query attention (GQA) heads. For the query axis, averaging can obscure rare but important query–key interactions, so \ourmethod{} instead takes the \textbf{maximum}, which preserves such outliers. This is supported by the heavy-tailed distribution in \cref{fig:dist_query_head} and by gains on the RULER benchmark (\cref{table:ruler_ablation}). For the GQA axis, we simply average scores across KV groups. Unlike queries, head-level importance is correlated \citep{bhojanapalli2021leveraging}, making the \textbf{mean} accurate and stable. Notably, if we normalize $K$ and $Q$ prior to computing the score, we can achieve the same average by \textit{pre-aggregation}: averaging normalized queries across KV groups due to the linearity of the mean and the outer product $QK^T$. Pre-aggregation also lowers computation and memory cost by a factor of the number of KV groups (which is large in most modern models), enabling compatibility with modern architectures and enhancing efficiency.

\subsection{Chunked Prefill with \ourmethod{}}
The three components described above—query subselection, scoring, and aggregation—are integrated
into the overall \emph{chunked prefill} process. For each incoming block $X_i$, \ourmethod{}
subselects the active KV tokens using the procedure in \cref{alg:ourmethod}. The resulting subset of keys and values is then passed to the attention computation for that chunk. By reducing the KV budget at each chunk, \ourmethod{} shrinks both compute and memory transfer costs. This enables significant prefill acceleration while maintaining accuracy across long-context benchmarks and is summarized in Algorithm~\ref{alg:geoselect_chunked_prefill}.

%% file: contents/results.tex
\section{results}

\vspace{-2mm}
\begin{figure*}[t]
     \centering
     \begin{subfigure}[b]{0.32\linewidth}
         \centering
         \includegraphics[width=\linewidth]{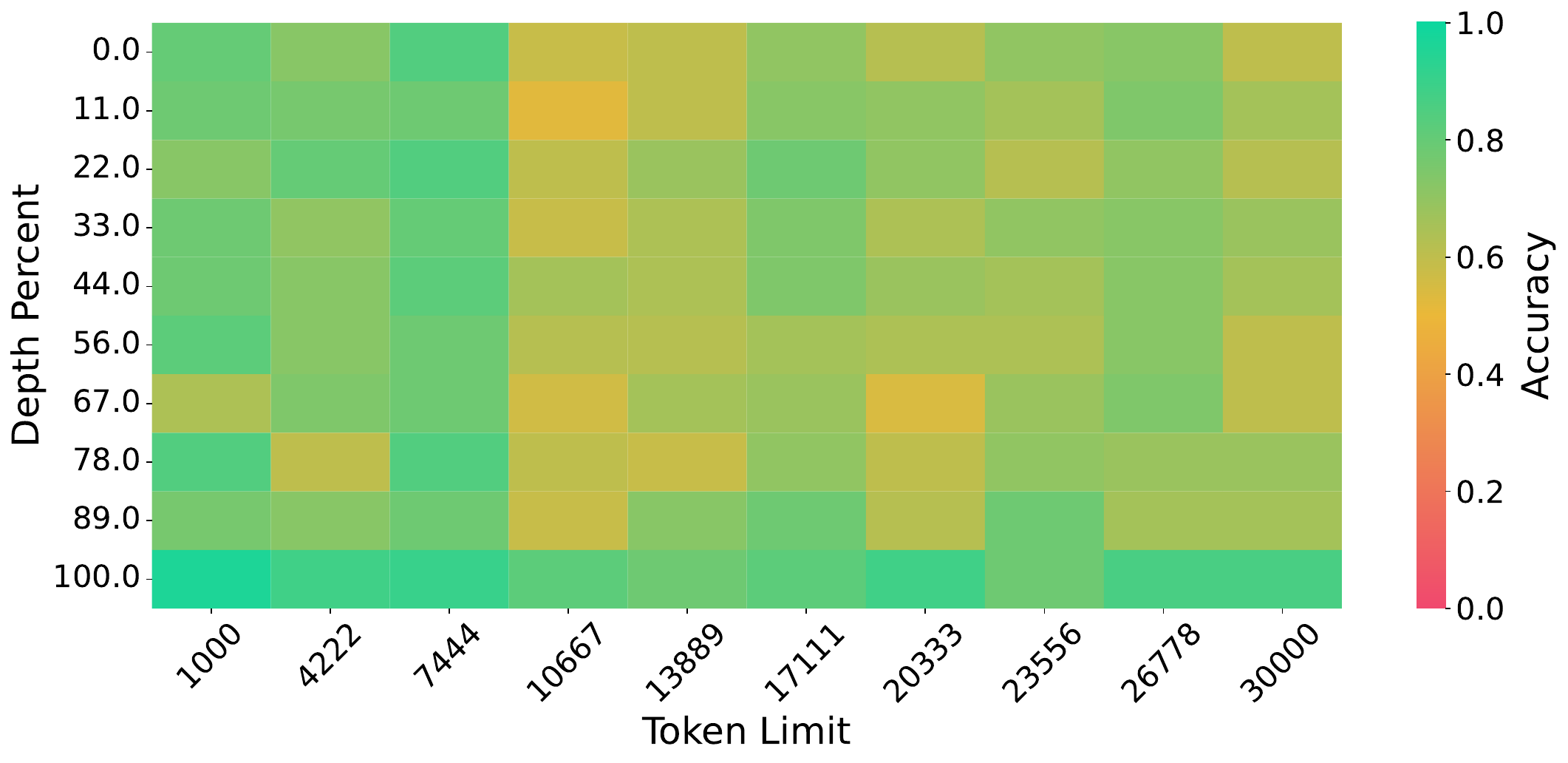}
         \caption{\ourmethod{}}
     \end{subfigure}
     \hfill
     \begin{subfigure}[b]{0.32\linewidth}
         \centering
         \includegraphics[width=\linewidth]{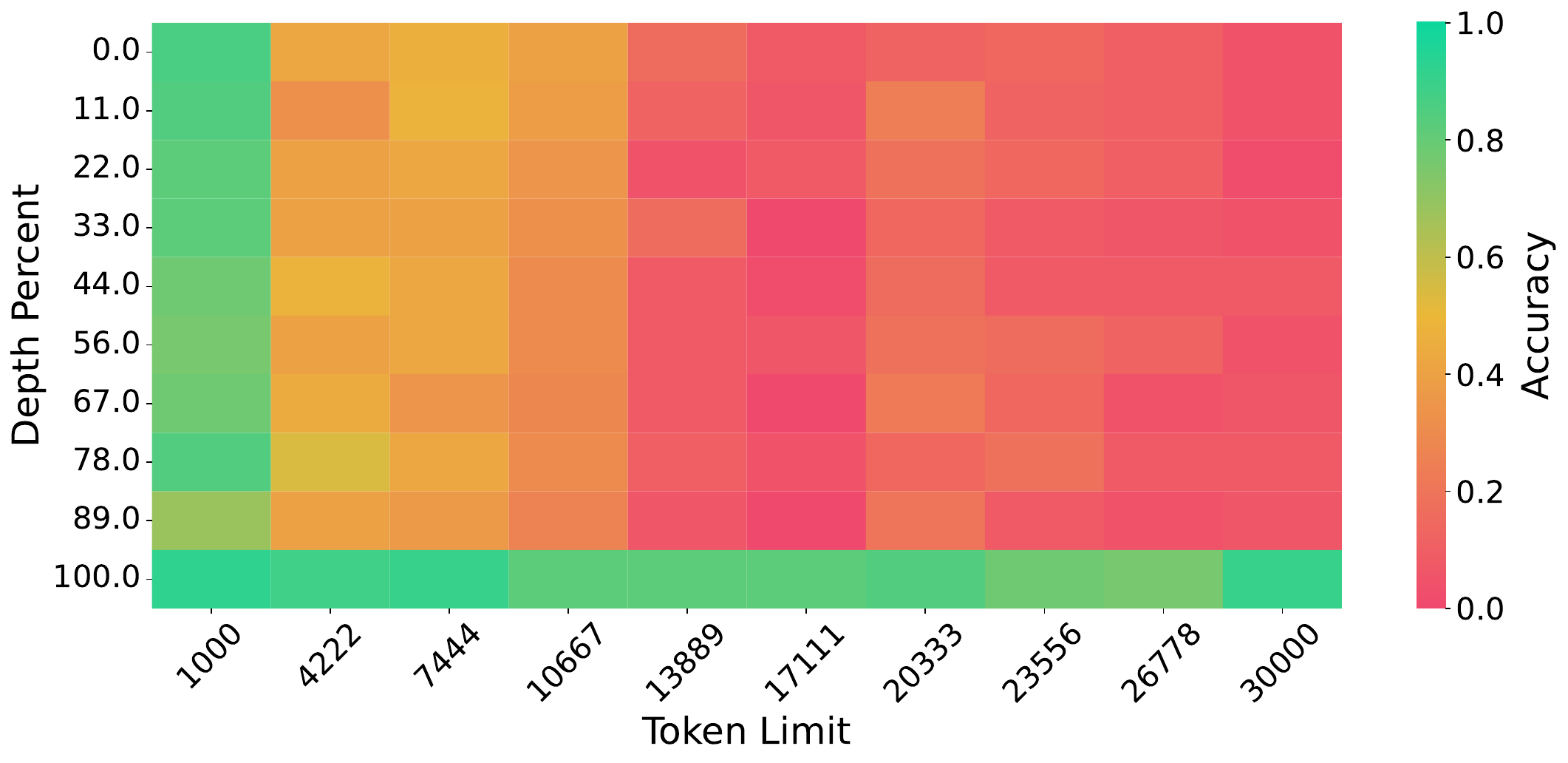}
         \caption{SampleAttention}
     \end{subfigure}
        \hfill
    \begin{subfigure}[b]{0.32\linewidth}
         \centering
         \includegraphics[width=\linewidth]{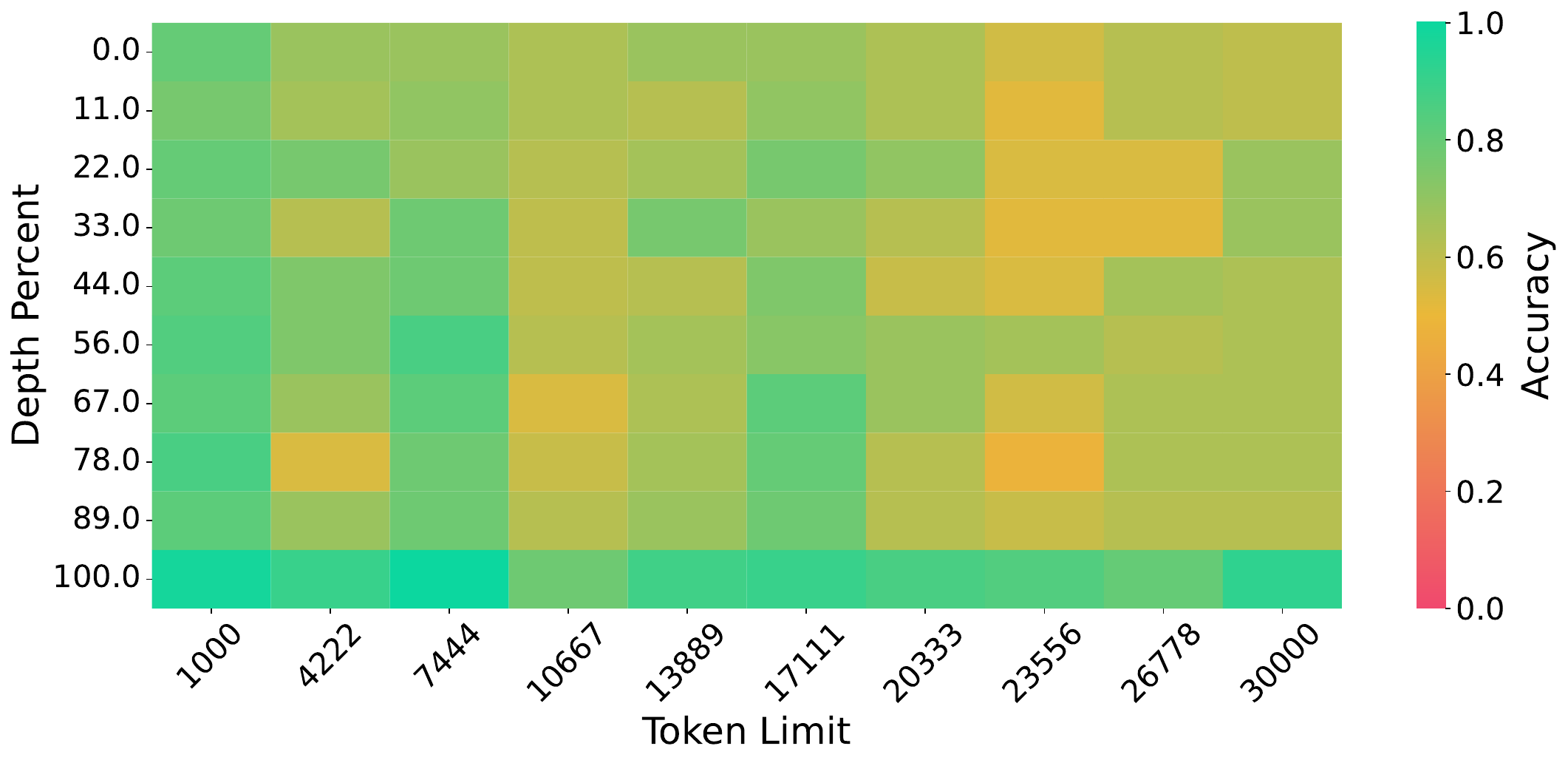}
         \caption{Full}
     \end{subfigure}
      \hfill
     \caption{Accuracy across document length and needle depth for NIAH with $B_\text{SA}=2048$ and $B_{\text{CP}}=128$. Results for additional sparse attention methods are shown in \cref{fig:needle_haystack_extended}.} 
    \label{fig:needle_haystack}
    % \vspace{-3mm}
\end{figure*}

In this section, we empirically demonstrate the effectiveness of \ourmethod{}. 
We begin with a detailed description of the evaluation setup and sparse attention baselines, then present our results.

{\bf{Sparse Attention Baselines:      }} We compared a group of sparse attention algorithms on a set of models across a variety of selective budgets. The competing algorithms include \textsc{SampleAttention} \citep{zhu2024sampleattention} which uniformly samples queries to compute $S$, \textsc{LessIsMore} \citep{yang2025moretrainingfreesparseattention} which computes $S$ only at specified layers, \textsc{SparQ} \citep{ribar2024sparq} which subselects along channel dimension before computing $S$, and \textsc{Loki} \citep{singhania2024loki} which down-projects keys and queries before computing $S$. Unless otherwise specified, \ourmethod{} and \textsc{SampleAttention} subselect $16$ queries at a time while \textsc{SparQ} and \textsc{Loki} down project keys and queries to $64$ dimensions.

{\bf{Models and Evaluation Setup:      }}
To evaluate the efficacy of \ourmethod{} across diverse model structures, we adopted \ourmethod{} and baseline methods across a variety of model families, including simple attention-only models, MoE-based FFN variants, and NoPE adaptations.  The models considered include Llama $3.2$-$3$B-Instruct \citep{dubey2024llama}, Qwen $2.5$-$3$B-Instruct \citep{yang2024qwen2technicalreport}, Smollm$3$ \citep{bakouch2025SmolLM3Smol}, Qwen$3$-$4$B, Qwen$3$-$30$B-A$3$B \citep{yang2025qwen3}, and GPT-OSS-$20$B \citep{agarwal2025gpt}. Additionally, we simulated a resource constrained scenario by utilizing chunked prefill with a block size $B_{\text{CP}}=128$. Unless otherwise specified, experiments were performed utilizing an Nvidia A$100$ GPU.

{\bf{Evaluation Summary:      }} 
We show that \ourmethod{} successfully preserves downstream accuracy on benchmarks involving long input prompts, extends naturally to generation-intensive tasks such as mathematical reasoning, and provides both attention module-level and end-to-end speedup. Results are summarized below:
\begin{itemize}[leftmargin=*, itemsep=.02em, topsep=0.5em]    
    \item \textbf{Needle-In-a-Haystack.} \ourmethod{} significantly outperforms competing selective attention methods on the needle-in-a-haystack (NIAH) benchmark (\cref{subsec:niah}).
    \item \textbf{RULER.} \ourmethod{} significantly outperforms competing methods on the RULER benchmark \citet{hsieh2024ruler}, achieving consistently better performance across all prompt lengths and models (\cref{subsec:ruler}).
    \item \textbf{LongBench.} \ourmethod{} exceeds the performance of competing baselines on the LongBench suite \citet{bai-etal-2024-longbench}, demonstrating substantially less performance degradation with decreasing selective budgets. Across all models and budgets, chunked prefill with \ourmethod{} outperforms other methods by at least $10$\textendash{}$20\%$ (\cref{subsec:longbench}).
    \item \textbf{Math500.} Although our primary focus is prefill, \ourmethod{} is also directly applicable to generation. On the Math500 benchmark, \ourmethod{} outperforms a sparse attention method specialized for generation and in some cases surpasses the accuracy of dense attention (\cref{subsec:math500}).
    \item \textbf{Ablation.} In sweeping over combinations of $B_{\text{CP}}$, $B_{\text{SA}}$, and $N_Q$ that govern the efficiency--accuracy trade-off in \ourmethod{}, we observe that accuracy decreases only gradually as sparsity increases.  This indicates that \ourmethod{} can be tuned for diverse hardware environments while maintaining high efficiency. (\cref{subsec:ablation})
    \item \textbf{Latency.} We measured time-to-first-token (TTFT) across increasing prompt lengths for end-to-end inference, along with standalone attention module latency. \ourmethod{} achieves a $5\times$ module-level speedup on $30$k tokens and a $3\times$ TTFT improvement on $50$k tokens relative to the base model with full attention, while scaling more efficiently than competing sparse methods (\cref{subsec:timing}). We provide additional runtime and memory complexity analysis for \ourmethod{} and the baselines in \cref{appendix:runtime-and-memory-complexity}.
\end{itemize}

% \vspace{-1mm}
\subsection{Needle In a Haystack}
\label{subsec:niah}
% \vspace{-1mm}

NIAH is a synthetic benchmark designed to evaluate how well large language models retrieve specific information from long contexts \citep{kamradt2023needle,liu2024lost}. A single \textit{needle} sentence containing a key fact is inserted into a large body of irrelevant text (i.e., \textit{haystack}), and the model is asked a question that can only be answered by finding that needle. We test Llama3.2-3B-Instruct with selective budget $B_\text{SA} = 2048$ and sequence lengths up to $30$K. 
\cref{fig:needle_haystack} shows that most selective attention methods incur significant degradation during chunked prefill. However, \ourmethod{} retains the ability to retrieve important information across various input sequence lengths.

% \vspace{-1mm}
\subsection{RULER Benchmark}

\input{tables/ruler_sa1024_column_major}
\label{subsec:ruler}
% \vspace{-1mm}
RULER \citep{hsieh2024ruler} is a synthetic test for evaluating long-context capabilities of LLMs using NIAH and QA tasks. It addresses shortcomings of NIAH by extending beyond basic retrieval to include multiple needles and new categories such as multi-hop tracing and aggregation, which require reasoning across dispersed information rather than locating a single token.

\input{tables/ruler_one_fourth}

We ran experiments to evaluate \ourmethod{} against sparse attention methods and the full attention baseline. \cref{tab:ruler_sa1024} reports results with $B_{\text{SA}} = 1024$, showing $10$\textendash{}$20\%$ higher scores than baselines. We also simulated $B_{\text{SA}}$ growing with the KV cache to maintain a constant compression ratio: during chunked prefill and token generation, $B_{\text{SA}}$ is set to $25\%$ of  full cache length. \cref{table:ruler_one_fourth} shows that accuracy loss remains very limited even at long sequences.

\subsection{LongBench}
\label{subsec:longbench}
LongBench \citep{bai-etal-2024-longbench} is a multi-task benchmark designed to evaluate the ability of LLMs in handling long sequences. It includes real-world tasks with inputs such as books, technical reports, and multi-document collections. The benchmark features substantial input lengths average 12K tokens and maximum exceeding 60K tokens measured with the Llama tokenizer, making it a rigorous test for models that aim to process extended inputs effectively.

\cref{table:longbench_compare} reports the normalized accuracy of sparse attention methods across various models and values of $B_\text{SA}$. Each cell shows relative scores compared to the dense baseline (where 1.0 indicates no accuracy drop). Across models, we observe that \ourmethod{} maintains minimal accuracy degradation even with a small budget (e.g., $B_\text{SA} = 1{,}024$). Moreover, \ourmethod{} consistently outperforms competing methods across both models and budget settings.

\input{tables/longbench_column_major}

% \vspace{-1mm}
\subsection{Math 500 Benchmark}
\label{subsec:math500}
% \vspace{-1mm}
Although our primary focus is prefill optimization, \ourmethod{} is also directly applicable to the generation phase, where the model computes attention for a single query and no query subselection is applied. To demonstrate its efficacy in this setting, we evaluated \ourmethod{} on a generation-intensive reasoning task, applying chunked prefill whenever applicable and using sparse attention for the computation. As shown in \cref{table:math500}, \ourmethod{} outperforms a sparse attention method specifically designed for generation, and in some cases even surpasses the accuracy of dense attention.
% \vspace{-1mm}
\subsection{Ablation study}
\label{subsec:ablation}
% \vspace{-1mm}
To examine the efficiency–accuracy trade-off of \ourmethod{}, we performed ablations over the key $B_{\text{SA}}$, $B_\text{CP}$ and $N_q$ parameters . \cref{table:longbench_compare,table:ruler_ourmethod_b_sa,table:longbench_compare_appendix} for both LongBench and Ruler varying $B_{\text{SA}}$ demonstrate that across models and benchmarks,  accuracy decreases very gradually as sparsity increases while latency and memory usage improve substantially. We achieve less than a $3\%$ drop in performance with less than $12\%$ of the original tokens used for attention. As shown in \cref{table:longbench_ablation_bpp}, \ourmethod{} maintains stable performance under varying $B_\text{CP}$. In \cref{table:longbench_ablation_nq} we observe that even with a small setting of $N_q = \tfrac{1}{16} B_\text{CP}$,  accuracy drops by only \textasciitilde $3\%$ compared to full attention. This robustness allows practitioners to tune \ourmethod{} for diverse hardware constraints with minimal quality loss.
% \vspace{-1mm}
\subsection{Attention Latency and TTFT}
\label{subsec:timing}
% \vspace{-1mm}
To evaluate the efficiency of \ourmethod{}, we measured latency in two settings: standalone attention modules and time-to-first-token (TTFT) on the Qwen$3$-$4$B model. All experiments use the HuggingFace implementation with \texttt{bfloat16} precision and FlashAttention for the dense baseline. Each data point is averaged over $100$ trials, and hyperparameters follow original publications. To compare \ourmethod{} with existing sparse attention methods, we measured performance across increasing input sequence lengths. All reported times are expressed as speedups relative to the full attention module, and in the end-to-end case, the base Qwen$3$-$4$B model with chunked prefill ($B_{\text{CP}}=128$). \cref{fig:standalone_timing,fig:end_to_end_timing} show \ourmethod{} scales substantially better than the dense baseline and consistently outperforms or matches the strongest sparse attention competitors on Nvidia A100. \ourmethod{} also achieves significant speedups on consumer-grade hardware: \cref{fig:standalone_timing_cpu,fig:standalone_timing_small_gpu} show $5$–$6\times$ gains at long context lengths on both an Intel Xeon W-2125 CPU and an Nvidia RTX 2080 GPU, and in decoding (\cref{fig:timing_plots_decode_math_500}).

\begin{figure*}
    \centering
    \begin{subfigure}[b]{0.45\textwidth}
        \centering
        \includegraphics[width=\textwidth]{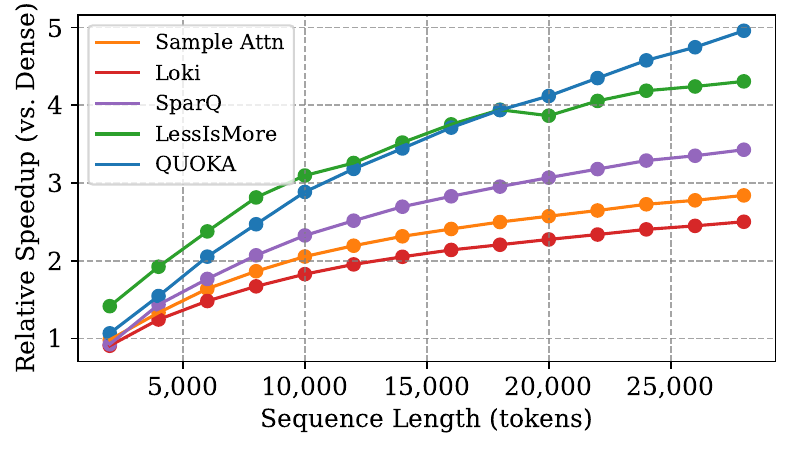}
        \caption{Attention latency on NVIDIA A100 GPU}
        \label{fig:standalone_timing}
    \end{subfigure}
    \hfill
    \begin{subfigure}[b]{0.475\textwidth}  
        \centering 
        \includegraphics[width=\textwidth]{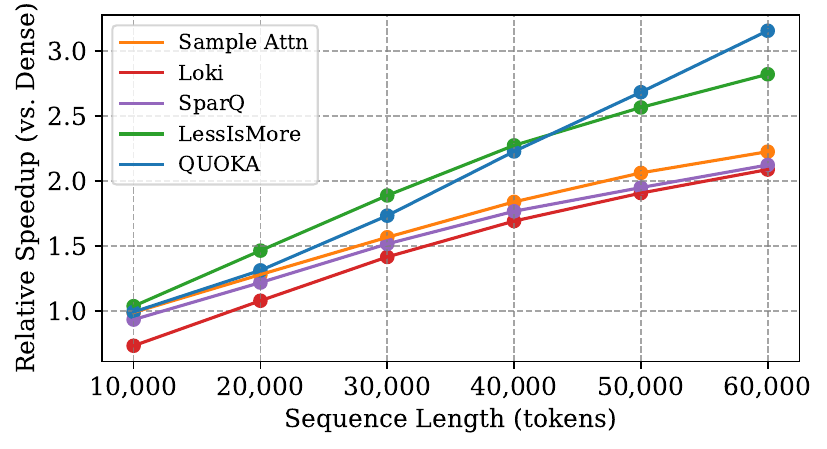}
        \caption{TTFT on NVIDIA A100 GPU}
        \label{fig:end_to_end_timing}
    \end{subfigure}
    % \vskip\baselineskip
    \begin{subfigure}[b]{0.475\textwidth}   
        \centering 
        \includegraphics[width=\textwidth]{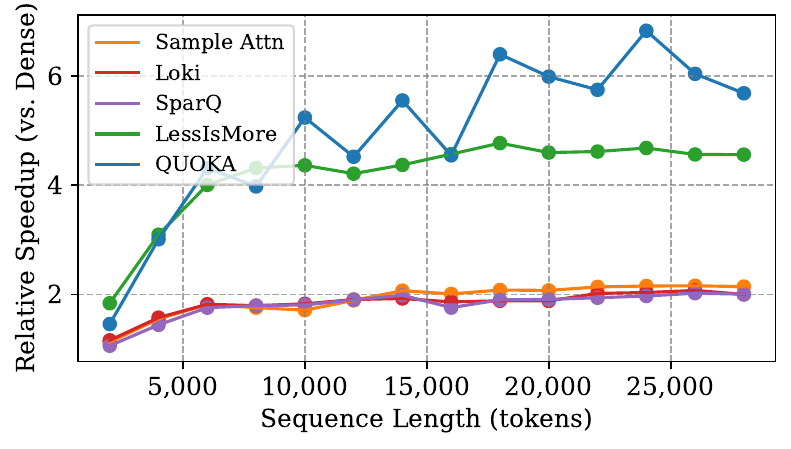}
        \caption{Attention latency on Intel Xeon W-2125 CPU}
        \label{fig:standalone_timing_cpu}
    \end{subfigure}
    \hfill
    \begin{subfigure}[b]{0.475\textwidth}   
        \centering 
        \includegraphics[width=\textwidth]{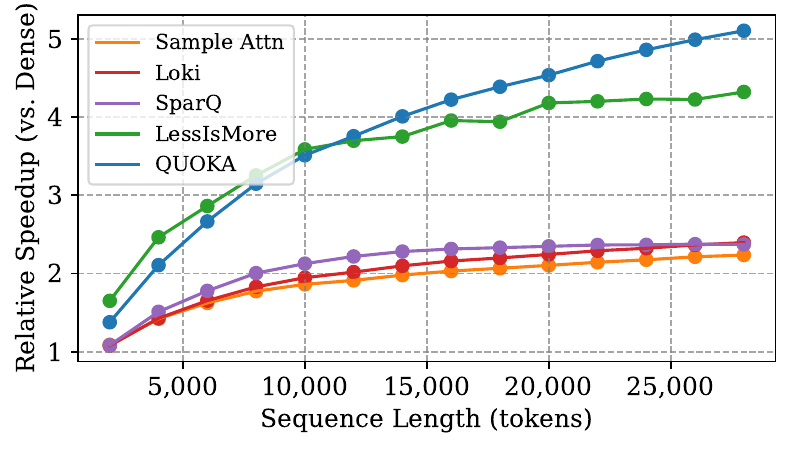}
        \caption{Attention latency on NVIDIA RTX 2080 GPU}
        \label{fig:standalone_timing_small_gpu}
    \end{subfigure}
    \caption{Relative speedup of attention and TTFT compared to dense attention baseline using $B_\text{CP}=128$ on different hardware.}
    \label{fig:timing_plots}
    
\end{figure*}

%% file: tables/ruler_sa1024_column_major.tex
\begin{table}[t]
\caption{{\textbf{RULER evaluation results across increasing lengths with $B_{\textrm{SA}}=1,024$}.} Highest per column in \textbf{bold} (Higher is better). Attention sparsification applied on full attention layers.}
\label{tab:ruler_sa1024}
% \footnotesize
\resizebox{\textwidth}{!}{%
\begin{tabular}{lcccccccccccccccccccc}
\toprule
& \multicolumn{4}{c}{\textit{Llama-3.2-3B-Instruct}} & \multicolumn{4}{c}{\textit{Qwen-2.5-3B-Instruct}} & \multicolumn{4}{c}{\textit{Qwen-3-4B}} & \multicolumn{4}{c}{\textit{Smollm3}} & \multicolumn{4}{c}{\textit{GPT-OSS-20B}}\\
\cmidrule(lr){2-5} \cmidrule(lr){6-9} \cmidrule(lr){10-13} \cmidrule(lr){14-17}  \cmidrule(lr){18-21}
Length & 4k & 8k & 16k & 32k & 4k & 8k & 16k & 32k & 4k & 8k & 16k & 32k & 4k & 8k & 16k & 32k & 4k & 8k & 16k & 32k \\
\midrule
SnapKV       & 29.15 & 13.70 & 12.44 &  6.21 & 19.25 & 12.40 &  9.88 &  8.13 & 27.29 & 17.15 & 10.94 & 10.07 & 43.72 & 35.98 & 24.21 & 12.23 & 58.18 & 31.45 & 24.82 & 16.63\\
KeyDiff      & 53.34 & 31.10 & 24.29 & 14.87 & 34.09 & 20.68 & 15.79 & 10.32 & 37.29 & 27.15 & 20.94 & 12.07 & 62.85 & 51.60 & 41.64 & 30.37& 69.58 & 28.76 & 15.86 & 9.65 \\
\midrule
LessIsMore   & 75.15 & 49.23 & 30.44 & 19.16 & 36.66 & 20.21 & 12.84 & 10.12 & 65.55 & 42.75 & 24.39 & 14.87 & 79.67 & 50.17 & 35.12 & 24.21 & 67.35 & 54.49 & 38.27 & 20.11\\
Loki         & 74.84 & 56.50 & 25.76 &  8.05 & 74.40 & 60.09 & 48.96 & 34.12 & 82.83 & 65.19 & 52.29 & 39.31 & 84.52 & 64.20 & 50.10 & 22.66 & 75.48 & 65.36 & 54.67 & 39.92\\
SparQ        & 79.36 & 60.80 & 48.59 & 31.14 & 78.07 & 59.87 & 54.71 & 36.74 & 87.93 & 68.97 & 56.02 & 35.20 & 82.45 & 57.62 & 32.18 & 18.69 & 70.07 & 54.00 & 30.75 & 15.20\\

SampleAttn   & 78.25 & 61.14 & 48.31 & 31.73 & 77.17 & 60.88 & 56.64 & 36.17 & 87.84 & 72.46 & 59.57 & 40.72 & 85.72 & 66.44 & 59.10 & 45.98 & 76.20& 70.35 & 53.91 & 30.42\\
\midrule
\ourmethod{} & \textbf{86.71} & \textbf{80.19} & \textbf{70.90} & \textbf{57.01} &
\textbf{87.85} & \textbf{74.27} & \textbf{66.82} & \textbf{59.37} &
\textbf{93.73} & \textbf{91.07} & \textbf{88.57} & \textbf{74.83} &
\textbf{89.97} & \textbf{79.94} & \textbf{72.69} & \textbf{61.37} & \textbf{78.92} & \textbf{79.19} & \textbf{73.40} & \textbf{57.79}\\
\bottomrule
\end{tabular}
}
\vspace{-5mm}
\end{table}

%% file: tables/ruler_one_fourth.tex
\begin{wraptable}{r}{0.45\linewidth}
\vspace{-3mm}
\centering
% \caption{{\bf{\ourmethod{} Ruler Performance}} \newline {\footnotesize{Results of the RULER benchmark at different test lengths for a set of models with full cache, and $1/4$ of the cache active in each attention module.}}}
\caption{\textbf{RULER eval of \ourmethod{}} with $B_\text{SA}$ set to 25\% of the KV cache length}
\label{table:ruler_one_fourth}
\resizebox{\linewidth}{!}{
\begin{tabular}{lc c ccc}
\toprule

\multirow{2.5}{*}{Model} & \multirow{2.5}{*}{Budget} & \multicolumn{4}{c}{Prompt Length} \\
\cmidrule(lr){3-6} 
& &
{4096} & {8192} & 
{16384} & {32768}
\\
\midrule
\multirow{2}{*}{\textit{Llama3.2-3B}} & Full & 87.50 & 81.33 & 78.98& 76.31\\
& $25\%$ & 86.94 & 79.72 & 76.02 & 74.14\\
\midrule
\multirow{2}{*}{\textit{Qwen2.5-3B}} & Full & 89.56 & 81.99 & 76.09 & 71.69\\
& $25\%$ & 86.07 & 79.78 & 74.25 & 68.84\\
\midrule
\multirow{2}{*}{\textit{Qwen3-4B}} & Full & 93.32 & 91.68 & 91.18& 88.54\\
& $25\%$ & 92.50 & 91.35 & 90.63 & 87.87\\
\midrule

% \multirow{2}{*}{\textit{Qwen3-30B-A3B-Inst}}
\multirow{2}{*}{\shortstack{\textit{Qwen3-30B}\\\textit{A3B-Instruct}}}
& Full & 94.08 & 93.75 & 92.02& 91.87\\
& $25\%$ & 93.25 & 92.77 & 91.90 & 91.08\\
\midrule
\multirow{2}{*}{\textit{Smollm3}} & Full & 91.12 & 83.46& 80.11& 75.18\\
& $25\%$ & 89.60 & 81.45 & 78.72 & 73.55\\
\midrule
\multirow{2}{*}{\textit{GPT-OSS-20B}} & Full & 79.35 & 79.32 & 79.78& 75.47\\ 
& $25\%$ & 77.40 & 80.66 & 77.17 & 73.45\\
\bottomrule
\end{tabular}
}
\vspace{-7mm}
\end{wraptable}

%% file: tables/longbench_column_major.tex
\begin{table*}[!t]
\centering
\caption{\textbf{LongBench results (Higher is better)}. For each model, the three columns correspond to selective budgets $B_{\text{SA}} \in \{512, 1024, 2048\}$.}
\label{table:longbench_compare}
\resizebox{\textwidth}{!}{
\begin{tabular}{l ccc ccc ccc ccc}
\toprule
& \multicolumn{3}{c}{\textit{Llama3.2-3B-Instruct}} 
& \multicolumn{3}{c}{\textit{Qwen2.5-3B-Instruct}}
& \multicolumn{3}{c}{\textit{Qwen3-4B}}
& \multicolumn{3}{c}{\textit{Smollm3}} \\
\cmidrule(lr){2-4} \cmidrule(lr){5-7} \cmidrule(lr){8-10} \cmidrule(lr){11-13}

Budget & 512 & 1024 & 2048 & 512 & 1024 & 2048 & 512 & 1024 & 2048 & 512 & 1024 & 2048 \\
\midrule
LessIsMore     & 0.703 & 0.788 & 0.850 
               & 0.461 & 0.556 & 0.659 
               & 0.665 & 0.773 & 0.868 
               & 0.765 & 0.842 & 0.918 \\
SparQ          & 0.721 & 0.802 & 0.842 
               & 0.636 & 0.726 & 0.808 
               & 0.686 & 0.782 & 0.872 
               & 0.725 & 0.805 & 0.922 \\
Loki           & 0.686 & 0.757 & 0.842 
               & 0.589 & 0.671 & 0.787 
               & 0.622 & 0.782 & 0.872 
               & 0.384 & 0.801 & 0.622 \\
SampleAttn     & 0.738 & 0.800 & 0.901 
               & 0.660 & 0.756 & 0.828 
               & 0.755 & 0.875 & 0.947 
               & 0.856 & 0.929 & 0.966 \\
\midrule
\ourmethod{}   & \textbf{0.945} & \textbf{0.972} & \textbf{0.986} 
& \textbf{0.869} & \textbf{0.945} & \textbf{0.977} 
& \textbf{0.966} & \textbf{0.992}  & \textbf{0.995} 
& \textbf{0.998} & \textbf{1.03}  & \textbf{1.028} \\
\bottomrule
\end{tabular}
}
\vspace{-2mm}
\end{table*}

%% file: contents/related_work.tex
\section{Related Work}

As discussed in \cref{sec:background}, the computational bottleneck in prefill stems from the quadratic growth of attention with the size of the KV cache. Under chunked prefill, this cost can be mitigated by subselecting the KV pairs used in the attention. In this context, we review three related lines of work.

{\bf{Dynamic query-dependent sparse attention:}}
Query-dependent sparse attention methods select a subset of cached keys for each query using lightweight proxies for attention scores \citep{singhania2024loki,tang2024quest,yang2024tidaldecode,ribar2024sparq}. While effective for single-query decode, naïvely averaging such proxies across multiple queries in a prefill chunk often degrades accuracy because these methods are tuned for generation-time settings. SampleAttention \citep{zhu2024sampleattention} targets prefill but treats multiple queries homogeneously. Since queries can exhibit distinct geometry with respective to keys, such a geometry-aware query and key subselection is desirable. \ourmethod{} first selects representative queries via cosine dissimilarity and then selects keys via cosine similarity, enabling higher sparsity at comparable accuracy.

{\bf{KV cache eviction:}}
KV cache eviction removes low-salience KV pairs to reduce memory footprint based on pre-defined patterns, attention scores, or cosine similarity of the keys \citep{xiaoefficient,zhang2024h2o,li2024snapkv,oren2024transformers,park2025keydiff}. Like query-dependent sparsification, most eviction policies are designed for single-query generation; a few works aggregate importance across multiple queries \citep{li2024snapkv,kim2024infinipot}, but typically treat queries homogeneously and remain generation-centric. Eviction is complementary to our approach and could be integrated with \ourmethod{} in chunked prefill; we leave this to future work.

{\bf{Kernel-level sparse attention:}}
Another line of work accelerates prefill through kernel-level sparse attention, typically relying on predefined patterns such as block, band, or strided sparsity \citep{child2019generating,dao2022flashattention,xu2025xattention,jiang2024minference,lai2025flexprefill}. While effective with specialized CUDA implementations, these methods require hardware-specific kernel optimizations and often incur additional overhead. In particular, under chunked prefill the repeated kernel invocations and memory transfers substantially reduce their efficiency. By contrast, \ourmethod{} is fully compatible with standard dense kernels and avoids such hardware and runtime dependencies.

%% file: contents/conclusion.tex
\section{Conclusion}
We introduced \ourmethod{}, a training-free and hardware-agnostic sparse attention mechanism that reduces LLM inference latency. Leveraging the geometry of queries and keys, \ourmethod{} identifies and retains the most influential queries to score and subselect the KV cache for attention without sacrificing accuracy. Across LongBench and RULER, \ourmethod{} achieves near-dense performance with a fraction of the budget while outperforming competing sparse attention methods. On Math500 \ourmethod{} further demonstrates versatility by surpassing generation-specific methods and even dense attention in some cases. In addition, \ourmethod{} provides substantial efficiency gains, including up to $5\times$ standalone attention speedup and $3\times$ TTFT reduction across different devices, suggesting that selective attention is a promising direction for TTFT and TPS optimization. An avenue for future improvement is to make the computation of $\bar{Q}K^\top$ more efficient, for example by exploiting query/key channel sparsity or learned low-dimensional projections.

%% file: contents/ethics_and_reproducibility.tex
\section{Ethics Statement}

This work presents \ourmethod{}, a hardware-agnostic sparse attention algorithm designed to accelerate transformer inference. Our research does not involve human subjects, sensitive personal data, or datasets that raise privacy or security concerns. All datasets used (Needle-In-A-Haystack, LongBench, RULER, and Math500) are publicly available and widely used in the research community. To our knowledge, \ourmethod{} does not introduce methodologies or applications that pose foreseeable risks of misuse or harm. We have no conflicts of interest or sponsorship to declare. We believe this work adheres to the code of ethics and does not raise any ethical concerns.

\section{Reproducibility Statement}

We have taken care to ensure that \ourmethod{} is readily reproducible. The main paper includes a detailed description of the algorithm, including its design principles, KV scoring/subselection strategy, and implementation using standard linear algebra routines. We provide comprehensive experimental setups, including datasets used, evaluation metrics, and hardware specifications. All hyperparameters and test configurations are documented in the main text and appendix. While we have not included source code, the simplicity of \ourmethod{}’s training-free design makes it straightforward to implement using widely available libraries.

%% file: contents/appendix.tex
\renewcommand \thepart{}
\renewcommand \partname{}

\newpage

\rule[0pt]{\columnwidth}{3pt}
\begin{center}
    \huge{\ourmethod{} \\
    Supplementary Material}
\end{center}
\vspace*{3mm}
\rule[0pt]{\columnwidth}{1pt}%\hline
\vspace*{-.5in}

% Making Table of contents for appendix only.
% https://tex.stackexchange.com/questions/419249/table-of-contents-only-for-the-appendix
\appendix
\addcontentsline{toc}{section}{}
\part{}
%\parttoc

% Appendix trick from
% https://ckadapa.wordpress.com/2019/09/24/formatting-equations-in-appendix-in-latex/
\renewcommand{\theequation}{A.\arabic{equation}}
% reset the counter
\setcounter{equation}{0}

% \appendix
% \section{Appendix}

\input{contents/appendix_chunked_prefill}
\input{contents/appendix_timing}

\input{contents/appendix_complexity}
\input{contents/appendix_proofs}

\input{contents/appendix_niah}
\input{contents/appendix_ruler}

\input{contents/appendix_longbench}

\input{contents/appendix_math_500}

\input{contents/appendix_ablations}

%% file: contents/appendix_chunked_prefill.tex
\section{Chunked Prefill with \ourmethod{}}

We provide the full algorithm of \ourmethod{} used in chunked prefill. \cref{alg:geoselect_chunked_prefill} summarizes the optimized prefill processes with \ourmethod{} under chunked prefill.

\input{algorithms/ourmethod_chunked_prefill}

%% file: algorithms/ourmethod_chunked_prefill.tex
\begin{algorithm}[h]
\caption{Chunked Prefill with \ourmethod{}}
\label{alg:geoselect_chunked_prefill}
\begin{algorithmic}[1]
\Require Input sequence $X = [X_0, X_1, \ldots, X_{B-1}]$, chunk size $L$, KV budget $B_{\text{SA}}$
\State Initialize output list $\mathcal{Y} \gets [\;]$
\For{each chunk $X_i$}
    \State Compute queries $Q_i$, keys $K_i$, values $V_i$ from $X_i$
    % \State $K_{\leq i} = [K_{<i} \mid K_i], \; V_{\leq i} = [V_{<i} \mid V_i]$ \Comment{Concatenate cache}
    \State $(\hat{K}, \hat{V}) \gets \textsc{QuoKA}(Q_i, K_{<i}, V_{<i}, B_{\text{SA}})$ \Comment{Subselect KV cache}
    \State $Y_i \gets \textsc{Attention}(Q_i, [\hat{K} \mid K_i], [\hat{V} \mid V_i])$
    \State Append $Y_i$ to $\mathcal{Y}$
    \State $K_{<i+1} \gets [K_{< i}|K_i], \; V_{<i+1} \gets [V_{<i}|V_i]$ \Comment{Update KV cache}
\EndFor
\State \Return $Y \gets \mathrm{Concat}(\mathcal{Y})$
\end{algorithmic}
\end{algorithm}

%% file: contents/appendix_timing.tex
\section{Additional Timing Results}

% All models were run using the HuggingFace implementation with \texttt{bfloat16} precision, and FlashAttention was enabled for the dense baseline. Experiments were performed on an Nvidia A$100$ GPU, with each data point averaged over $100$ trials. Hyperparameters for all methods were selected according to the recommendations of the original publications.

\begin{figure*}[htbp!]
     \centering
     \begin{subfigure}[b]{0.45\linewidth}
         \centering
         \includegraphics[width=\linewidth]{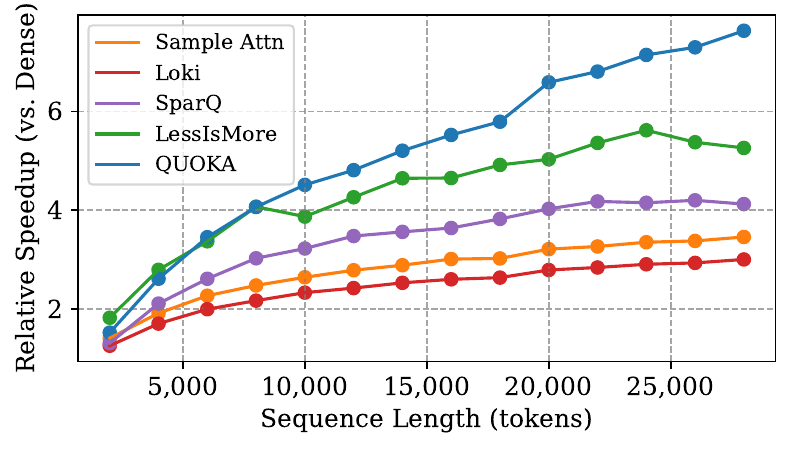}
         \caption{Attention latency (decode) on NVIDIA A100}
         \label{fig:standalone_timing_math_500}
     \end{subfigure}
     \hfill
     \begin{subfigure}[b]{0.45\linewidth}
         \centering
         \includegraphics[width=\linewidth]{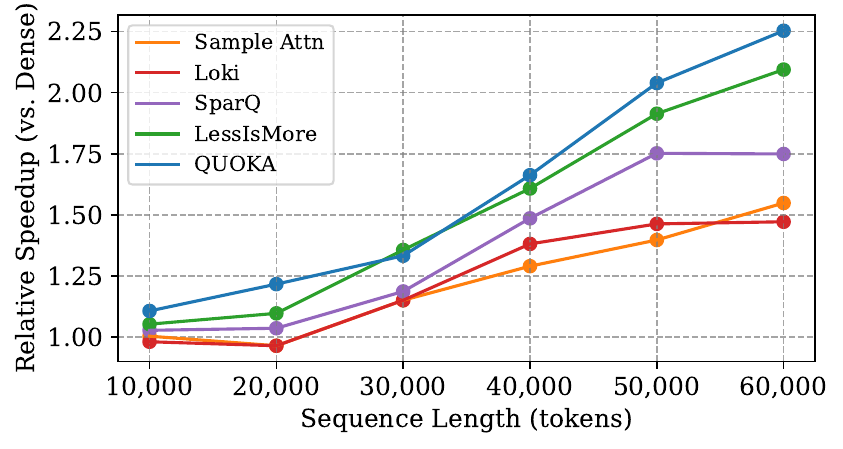}
         \caption{End-to-End latency (decode) on NVIDIA A100}
         \label{fig:end_to_end_timing_math_500}
     \end{subfigure}
     
     % \vspace{-2mm}
     \caption{Speedup versus full attention of an increasing number of decoding steps for a standalone attention module and end to end model (Qwen3-4B). Timings were measured on an Nvidia A100 GPU across 100 trials.}
    \label{fig:timing_plots_decode_math_500}
    % \vspace{-5mm}
\end{figure*}

% \begin{figure*}[htbp!]
%      \centering
%      \begin{subfigure}[b]{0.49\linewidth}
%          \centering
%          \includegraphics[width=\linewidth]{figures/PlotCPULatencyConsumer.pdf}
%          \caption{Attention Speedup on Intel Xeon W-2125 CPU}
%          \label{fig:standalone_timing_math_500}
%      \end{subfigure}
%      \hfill
%      \begin{subfigure}[b]{0.49\linewidth}
%          \centering
%          \includegraphics[width=\linewidth]{figures/PlotGPULatencyConsumer.pdf}
%          \caption{Attention Speedup on RTX2080}
%          \label{fig:end_to_end_timing_math_500}
%      \end{subfigure}     
%      % \vspace{-2mm}
%      \caption{\textbf{Speedup compared to full attention on Intel Xeon W-2125 CPU and RTX 2080 GPU.} The consistent superiority of \ourmethod{} highlights its portability and effectiveness across both CPU and GPU platforms.}
%     \label{fig:timing_plots_consumer_math_500}
%     % \vspace{-5mm}
% \end{figure*}

%% file: contents/appendix_complexity.tex
\section{Runtime and Memory Complexity}
\label{appendix:runtime-and-memory-complexity}

% $n_{\mathrm{KV}}$, 
%          number of attention heads $n_{\mathrm{Q}}$

We compute the runtime and memory complexity for \ourmethod{} and competing sparse attention methods in \cref{table:runtime-and-memory}.
For prefill chunk size $B_\text{CP}$, KV cache length $T$, number of $KV$ heads $n_{\mathrm{KV}}$, number of heads $n_\mathrm{Q}$, hidden dimension $d$ and number of subselected queries $N_\mathrm{Q}$, \ourmethod{} has $\mathcal{O}(B_\text{CP}+(N_Q(1+dn_{\mathrm{KV}}))T) $ runtime complexity and $\mathcal{O}(n_{\mathrm{KV}}N_QT)$ memory complexity. Note that both terms require $n_{\mathrm{KV}}$ rather $n_\mathrm{Q}$ where $n_\mathrm{Q} > n_{\mathrm{KV}}$, resulting in significant compute and memory savings. For SampleAttention, given the same parameters we obtain $\mathcal{O}(( d n_\mathrm{Q} + n_\mathrm{Q}/n_{\mathrm{KV}} + n_{\mathrm{KV}})N_QT)$ runtime complexity and $\mathcal{O}(n_\mathrm{Q}N_QT)$ memory complexity. Note that since this method computes attention logits before aggregation, $n_\mathrm{Q}$ appears in both terms. For Loki and SparQ, let $d_l < d$ be the lower dimension used to downproject the channels of $Q$ and $K$. With this SparQ has runtime complexity $\mathcal{O}(B_\text{CP} T d_l n_\mathrm{Q})$ and memory complexity $\mathcal{O}(n_\mathrm{Q}B_\text{CP}T)$ whereas Loki has runtime complexity $\mathcal{O}(d_l n_\mathrm{Q} (B_\text{CP} T + d (B_\text{CP} + T)))$ due to the matrix multiplications and memory complexity $\mathcal{O}(n_\mathrm{Q}B_\text{CP}T)$. Note that Loki additionally has memory overhead of $\mathcal{O}(d d_l n_\mathrm{Q})$ to store the downprojection matrices in each layer. Finally, since LessIsMore is not applied uniformly across layers, where $L$ is the number of layers, we amortize complexity to obtain runtime complexity of $\mathcal{O}(dn_\mathrm{Q}B_\text{CP}T/L)$ and memory complexity of $\mathcal{O}(n_\mathrm{Q}B_\text{CP}T/L)$. Note that, asymptotically our method is more efficient than other methods due to the fact that $n_{\mathrm{KV}} < n_\mathrm{Q}$, justifying our pre-aggregation design.

\begin{table}[h]
\caption{\textbf{Runtime and memory complexity of sparse attention methods.}}
\label{table:runtime-and-memory}
\begin{center}
\begin{tabular}{lll} 
\toprule
 & Runtime Complexity & Memory Complexity \\
\midrule
\ourmethod{} & $\mathcal{O}(B_\text{CP}+(N_Q(1+dn_{\mathrm{KV}}))T) $ & $\mathcal{O}(n_{\mathrm{KV}}N_QT)$ \\ 
SampleAttention & $\mathcal{O}(( d n_\mathrm{Q} + n_\mathrm{Q}/n_{\mathrm{KV}} + n_{\mathrm{KV}})N_QT)$ & $\mathcal{O}(n_\mathrm{Q}N_QT)$ \\ 
SparQ & $\mathcal{O}(B_\text{CP} T d_l n_\mathrm{Q})$ & $\mathcal{O}(n_\mathrm{Q}B_\text{CP}T)$ \\ 
Loki & $\mathcal{O}(d_l n_\mathrm{Q} (B_\text{CP} T + d (B_\text{CP} + T)))$ & $\mathcal{O}(n_\mathrm{Q}B_\text{CP}T)$ \\ 
LessIsMore & $\mathcal{O}(dn_\mathrm{Q}B_\text{CP}T/L)$ & $\mathcal{O}(n_\mathrm{Q}B_\text{CP}T/L)$ \\ 
\bottomrule
\end{tabular}
\end{center}
\end{table}

%% file: contents/appendix_proofs.tex
\section{Mathematical Proofs}
\label{sec:proof}
We provide the statement of \cref{theorem:keydiff-key-query} for completeness below along with the accompanying proof. Note that due to the unique geometry of LLM attention heads, this theorem implies that if $q_0$ is angularly distant from $M_Q$, it will have higher attention scores with most keys. Specifically, as $\textrm{CosSim}(M_Q, k) \rightarrow -1$, the cosine similarity $\textrm{CosSim}(M_Q, {q^*})$ is bounded above by a monotonically decreasing function converging to $-\frac12$.

\newtheorem*{thm1:appendix}{Theorem 1}
\begin{thm1:appendix}
Consider tokens a fixed query $q_0$ and key $k$, and let the average of a set of queries be denoted $M_Q$. Suppose $\textrm{CosSim}({k}, q_0) = \beta_k > 0$ and $\textrm{CosSim}(M_Q, k) = \alpha_k < 0$. Then 
\begin{equation}
\textrm{CosSim}(M_Q, {q^*}) \leq 1 + \alpha_k \beta_k - 0.5\alpha_k^2 -0.5\beta_k^2.
% No need for label since it's already labeled in the method_trimmed.tex file
%\label{eq:key-query-eq} 
\end{equation}
\end{thm1:appendix}

\begin{proof}
First compute the cosine similarity of $M_Q$ with $q^*$: $\textrm{CosSim}(M_Q, {q^*}) = \frac{{q^*}^\top M_Q}{||{q^*}|| ||M_Q||}$. We can expand $M_Q$ in an orthonormal basis containing $k/||k||$, $\{k/||k||,r_1, ..., r_{n-1}\}$ where $M_Q = ||M_Q|| \left(\alpha_q  k/||k|| +  \sum\limits_{i=1}^{n-1} \alpha_i r_i \right)$ and $\alpha_i = \textrm{CosSim}(M_Q, r_i)$. Let $\beta_i = \textrm{CosSim}({q^*}, r_i)$ and note since $r_i$ are orthonormal we have that $\alpha_k^2 + \sum\limits_{i=1}^{n-1}\alpha_i^2 = 1$ and that $\beta_k^2 + \sum\limits_{i=1}^{n-1}\beta_i^2 = 1$. Then

\begin{align*}
    \frac{{q^*}^\top M_Q}{||{q^*}|| ||M_Q||} & = \frac{{q^*}^\top \left(||M_Q||\alpha_k  k/||k|| +  \sum\limits_{i=1}^{n-1} ||M_Q|| \alpha_i r_i \right)}{||{q^*}|| ||M_Q||} \\
    &= \alpha_k \beta_k + \frac{1}{||{q^*}||} \sum\limits_{i=1}^{n-1} \alpha_i {q^*}^\top r_i \\
    & = \alpha_k \beta_k + \sum\limits_{i=1}^{n-1} \alpha_i \beta_i \\
    & \leq \alpha_k \beta_k + \sum\limits_{i=1}^{n-1} |\alpha_i| |\beta_i| \\
\intertext{Using Young's inequality:}
    & \leq \alpha_k \beta_k + \frac{1}{2}\sum\limits_{i=1}^{n-1} \alpha_i^2 + \beta_i^2 \\
    & = \alpha_k \beta_k +  \frac{1}{2}(1-\alpha_k^2) + \frac{1}{2}(1-\beta_k^2)\\
    & = 1 + \alpha_k \beta_k - \frac{1}{2}\alpha_k^2 -\frac{1}{2}\beta_k^2 \\
\end{align*}
\end{proof}

%% file: contents/appendix_niah.tex
\section{Additional Needle In a Haystack Experiments}

In \cref{fig:needle_haystack_extended} we show the results of the Needle In a Haystack experiment utilizing a variety of selective attention methods. We used the Llama3.2-3B-Instruct model and chunked prefill with $B_\text{CP}= 128$ with selective budget $B_\text{SA}=2048$. The figure clearly shows the degradation incurred using most selective attention techniques with chunked prefill, while \ourmethod{} is competitive even with full attention. This indicates that information is retained and recalled significantly more effectively using \ourmethod{} under chunked prefill. 

\begin{figure*}[t]
     \centering
     \begin{subfigure}[b]{0.32\linewidth}
         \centering
         \includegraphics[width=\linewidth]{figures/haystack_plot_SPARK.pdf}
         \caption{\ourmethod{}}
     \end{subfigure}
     \hfill
     \begin{subfigure}[b]{0.32\linewidth}
         \centering
         \includegraphics[width=\linewidth]{figures/haystack_plot_SampleAttention.pdf}
         \caption{SampleAttention}
     \end{subfigure}
        \hfill
    \begin{subfigure}[b]{0.32\linewidth}
         \centering
         \includegraphics[width=\linewidth]{figures/haystack_plot_full.pdf}
         \caption{Full}
     \end{subfigure}
      \hfill

      % \vspace{2em}  % space between the two rows
    
      % --- Bottom row: 2 subfigures centered as a block ---
      \begin{subfigure}[b]{0.32\linewidth}
         \centering
         \includegraphics[width=\linewidth]{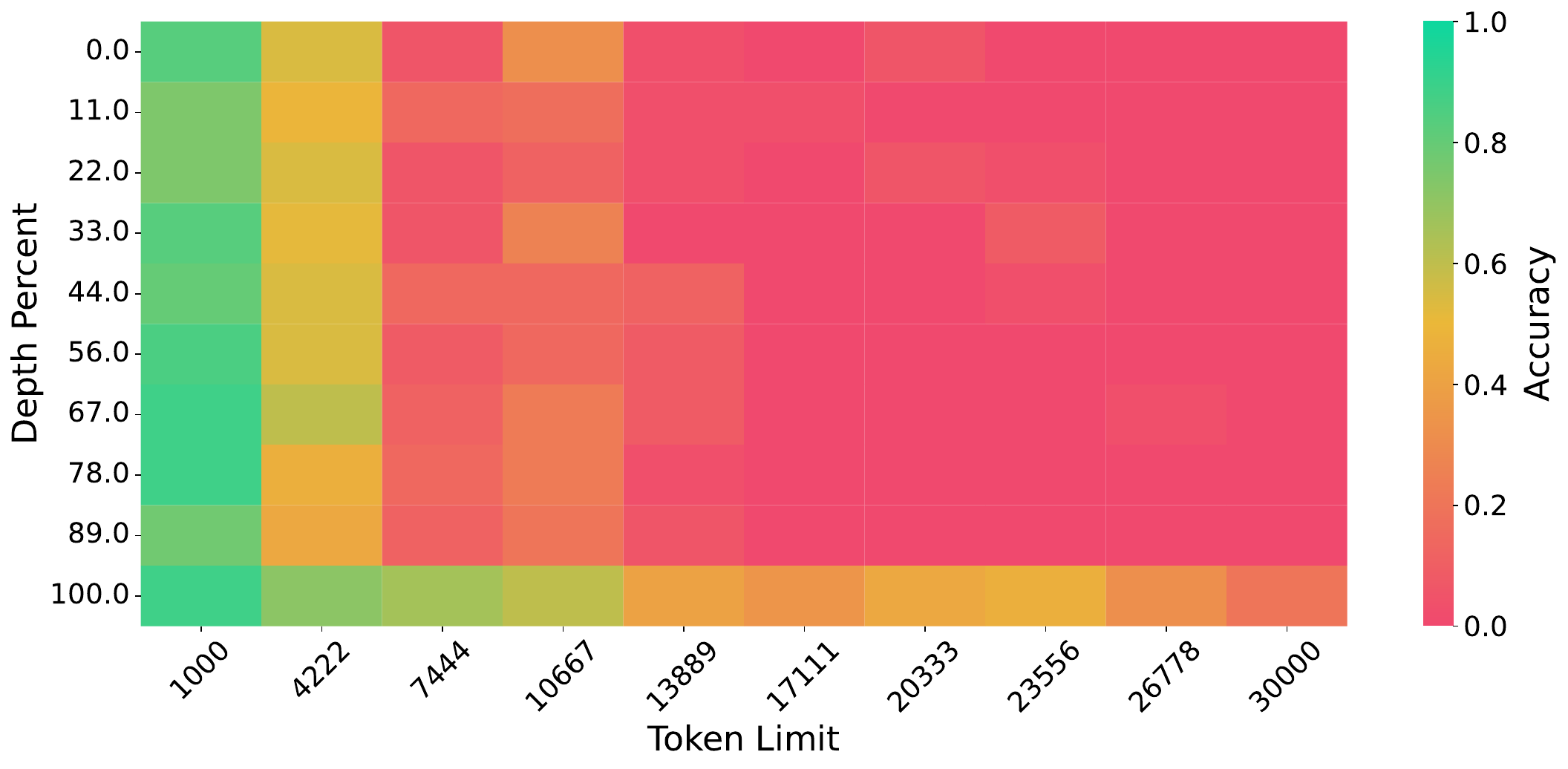}
         \caption{Loki}
     \end{subfigure}
     \hfill
     \begin{subfigure}[b]{0.32\linewidth}
         \centering
         \includegraphics[width=\linewidth]{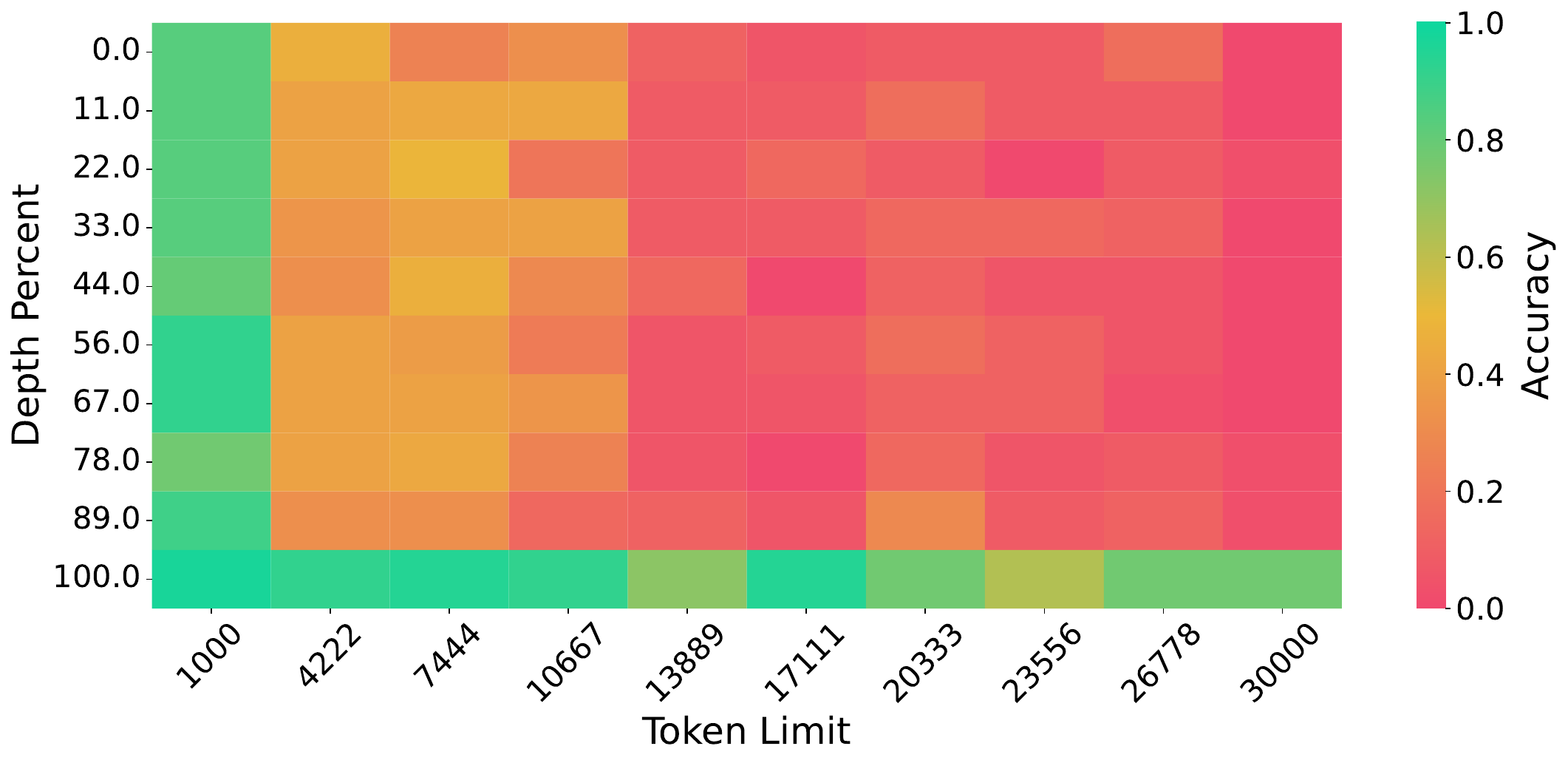}
         \caption{SparQ}
     \end{subfigure}
        \hfill
      \begin{subfigure}[b]{0.32\linewidth}
         \centering
         \includegraphics[width=\linewidth]{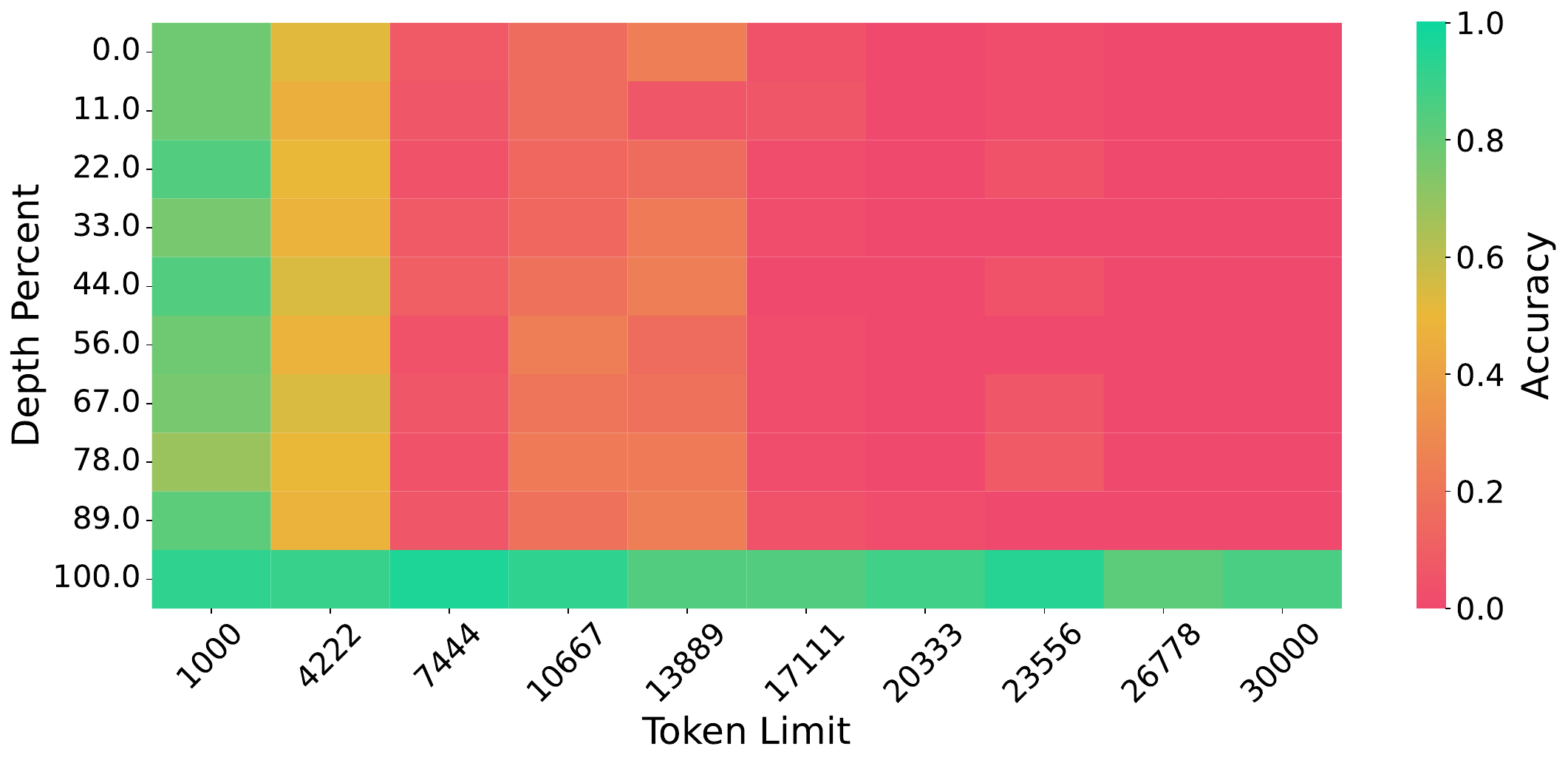}
         \caption{LessIsMore}
     \end{subfigure}
      \hfill
     \caption{Accuracy across document length and needle depth for NIAH using additional sparse attention methods with $B_\text{SA}=2048$ and $B_{\text{CP}}=128$.} 
    \label{fig:needle_haystack_extended}
    \vspace{-3mm}
\end{figure*}

%% file: contents/appendix_ruler.tex
\section{Additional RULER results}
We provide additional results of \ourmethod{} on the RULER benchmark, varying prompt length and selective attention budget $B_\text{SA}$. As shown in \cref{table:ruler_ourmethod_b_sa}, \ourmethod{} achieves near-dense accuracy even with $1/8 - 1/4$ of the full KV cache. This suggests that a significant speed up can be obtained for long context tasks with relative little performance drop.

\input{tables/ruler_ggattn}

%% file: tables/ruler_ggattn.tex
\begin{table*}[htbp!]
\caption{\textbf{\ourmethod{} RULER results} across prompt lengths and $B_\text{SA} \in \{1024, 2048, 4096\}$.}
\label{table:ruler_ourmethod_b_sa}
\begin{center}
\resizebox{.50\textwidth}{!}{
\begin{tabular}{l c ccccc}
\toprule

\multirow{2.5}{*}{Model} & \multirow{2.5}{*}{Budget} & \multicolumn{4}{c}{Prompt Length} \\
\cmidrule(lr){3-6} 
& &
{4096} & {8192} & 
{16384} & {32768}
\\
\midrule
\multirow{4}{*}{Llama-3.2-3B}
& Full & 87.50 & 81.33 & 76.98 & 74.31\\
& 4096 & 87.51 & 80.60 & 75.69 & 69.95\\
& 2048 & 87.33 & 80.20 & 73.79 & 63.67\\
& 1024 & 86.71 & 80.19 & 70.90 & 57.01\\
\midrule

\multirow{4}{*}{Llama-3.1-3B}
& Full & 89.45 & 90.24 & 90.53 & 86.65\\
& 4096 & 89.45 & 89.03 & 87.46 & 81.33\\
& 2048 & 89.20 & 88.82 & 86.21 & 78.86 \\
& 1024 & 89.24 & 89.51 & 83.25 & 70.04 \\
\midrule

\multirow{4}{*}{Qwen-2.5-3B}
& Full & 89.57 & 81.99 & 76.09 & 71.69\\
& 4096 & 89.53 & 81.53 & 74.83 & 68.43\\
& 2048 & 89.56 & 79.41 & 71.13 & 63.82\\
& 1024 & 87.85 & 74.27& 66.82& 59.37\\
\midrule

\multirow{4}{*}{Qwen-3-4B}
& Full & 93.32 & 91.68 & 91.18 & 88.54\\
& 4096 & 93.32 & 91.83 & 91.06 & 87.68\\
& 2048 & 93.13 & 91.80 & 88.87 & 85.23\\
& 1024 & 93.73 & 91.07 & 88.57 & 74.83\\
\midrule

\multirow{4}{*}{Smollm3}
& Full & 91.21 & 83.46 & 80.11 & 75.18\\
& 4096 & 91.21 & 83.70 & 79.66 & 71.83\\
& 2048 & 90.96 & 83.06 & 78.26 & 67.65\\
& 1024 & 89.97 & 79.94 & 72.69 & 61.37\\
\bottomrule

\end{tabular}
}

\end{center}
\end{table*}

%% file: contents/appendix_longbench.tex
\section{Extended LongBench Results}

In \cref{table:longbench_compare_appendix,table:longbench_compare_appendix2} we present the performance on individual tasks of different models using different sparse attention methods across different selective attention budgets $B_\text{SA}$. In particular, \cref{table:longbench_compare_appendix} shows detailed task scores for different sparse attention methods across an array of models in which \ourmethod{} consistently outperforms by a gap of $10-40\%$. \cref{table:longbench_compare_appendix2} show a more detailed ablation studying the effects of varying $B_\text{SA}$ on LongBench tasks using only \ourmethod{}. We can see that in general, a selective budget of $1024$ achieves less than $5\%$ error overall, which is remarkable considering this is less than $1/10$ of the average task length for the considered LongBench tasks.

\input{tables/longbench_compare_appendix}
\newpage

\input{tables/longbench_ggattn_appendix}
\newpage

%% file: tables/longbench_compare_appendix.tex
\begin{table*}[!t]
\caption{\textbf{LongBench results} comparing performance of different selective attention mechanisms across different budgets.}
\label{table:longbench_compare_appendix}
\vspace{-5mm}
\begin{center}
\resizebox{\textwidth}{!}{
\begin{tabular}{ccccccccccccccccccc}
\toprule

%%%%%%%%%%%%%%%%%%%%%%%%%%%%%%%%%%%%%%%%%%%%%%%%%%%%%%%%
%%%%%%%%%%%%%%%%%%%%% Header Rows %%%%%%%%%%%%%%%%%%%%%&
%%%%%%%%%%%%%%%%%%%%%%%%%%%%%%%%%%%%%%%%%%%%%%%%%%%%%%%%
&
&
\multicolumn{3}{c}{Single Doc. QA} &
\multicolumn{3}{c}{Multi Doc. QA} &
\multicolumn{3}{c}{Summarization} &
\multicolumn{3}{c}{Few$\-$shot Learning} &
\multicolumn{1}{c}{Synthetic} &
\multicolumn{2}{c}{Code} &
\\

\cmidrule(lr){3-5}
\cmidrule(lr){6-8}
\cmidrule(lr){9-11}
\cmidrule(lr){12-14}
\cmidrule(lr){15-15}
\cmidrule(lr){16-17}

&
& 
{Narrative QA} & {Qasper} & {MF-en} & 
{HotpotQA} & {2WikiMQA} & {Musique} & 
{GovReport} & {QMSum} & {MultiNews} & 
{TREC} & {TriviaQA} & {SAMSum} &  {PR-en} &
{Lcc} & {RB-P} & 
{Avg.} 
\\

\midrule
\midrule
\multicolumn{2}{c}{llama3.2-3B-Instruct}  & 22.91 & 40.49 & 49.99 & 50.96 & 43.29 & 26.97 & 33.42 & 24.28 & 24.98 & 73.5 & 90.17 & 42.04 & 96.0 & 56.53 & 56.95 & 48.832 \\
 
\midrule 

\multirow{3}{*}{\ourmethod{}}
 &  512 & 0.844 & 0.98 & 0.979 & 0.978 & 0.964 & 0.765 & 0.952 & 0.864 & 1.042 & 0.973 & 1.0 & 0.995 & 0.766 & 1.074 & 1.003 & 0.945  \\ 
 &  1024 & 0.867 & 1.001 & 0.995 & 1.005 & 1.005 & 0.786 & 0.978 & 0.908 & 1.034 & 1.007 & 1.005 & 1.003 & 0.927 & 1.034 & 1.021 & 0.972  \\ 
 &  2048 & 0.973 & 0.991 & 1.029 & 1.019 & 0.982 & 0.863 & 0.984 & 0.96 & 1.004 & 1.007 & 1.006 & 0.987 & 0.969 & 1.016 & 0.996 & 0.986  \\ 
 
\midrule 

\multirow{3}{*}{SampleAttention}
 &  512 & 0.759 & 0.66 & 0.647 & 0.751 & 0.645 & 0.276 & 0.917 & 0.787 & 1.028 & 0.667 & 0.983 & 0.977 & 0.12 & 0.962 & 0.896 & 0.738  \\ 
 &  1024 & 0.814 & 0.817 & 0.79 & 0.764 & 0.694 & 0.502 & 0.953 & 0.832 & 1.019 & 0.741 & 0.982 & 0.971 & 0.234 & 0.972 & 0.916 & 0.8  \\ 
 &  2048 & 0.919 & 0.922 & 0.958 & 0.94 & 0.801 & 0.753 & 0.963 & 0.915 & 1.008 & 0.864 & 1.0 & 1.006 & 0.531 & 0.995 & 0.942 & 0.901  \\ 
 
\midrule 

\multirow{3}{*}{TidalDecode}
 &  512 & 0.697 & 0.696 & 0.463 & 0.618 & 0.557 & 0.352 & 0.724 & 0.808 & 0.975 & 0.83 & 0.906 & 0.936 & 0.104 & 0.998 & 0.887 & 0.703  \\ 
 &  1024 & 0.776 & 0.822 & 0.633 & 0.703 & 0.704 & 0.595 & 0.782 & 0.864 & 1.004 & 0.844 & 0.97 & 0.962 & 0.188 & 1.016 & 0.956 & 0.788  \\ 
 &  2048 & 0.674 & 0.921 & 0.768 & 0.857 & 0.87 & 0.649 & 0.858 & 0.87 & 1.003 & 0.946 & 0.989 & 0.99 & 0.344 & 1.013 & 1.001 & 0.85  \\ 
 
\midrule 

\multirow{3}{*}{SparQ}
 &  512 & 0.679 & 0.607 & 0.635 & 0.659 & 0.604 & 0.346 & 0.909 & 0.787 & 1.045 & 0.68 & 0.948 & 0.973 & 0.089 & 0.983 & 0.88 & 0.721  \\ 
 &  1024 & 0.75 & 0.801 & 0.779 & 0.734 & 0.795 & 0.566 & 0.936 & 0.817 & 1.027 & 0.755 & 0.985 & 0.976 & 0.224 & 0.977 & 0.907 & 0.802  \\ 
 &  2048 & 0.909 & 0.927 & 0.964 & 0.932 & 0.773 & 0.748 & 0.967 & 0.898 & 1.002 & 0.844 & 0.986 & 0.996 & 0.536 & 0.989 & 0.923 & 0.893  \\ 
 
\midrule 

\multirow{3}{*}{Loki}
 &  512 & 0.57 & 0.662 & 0.634 & 0.662 & 0.575 & 0.211 & 0.891 & 0.803 & 1.038 & 0.646 & 0.908 & 0.933 & 0.052 & 0.927 & 0.783 & 0.686  \\ 
 &  1024 & 0.878 & 0.799 & 0.683 & 0.687 & 0.649 & 0.371 & 0.923 & 0.831 & 1.022 & 0.707 & 0.947 & 0.952 & 0.104 & 0.967 & 0.835 & 0.757  \\ 
 &  2048 & 0.907 & 0.897 & 0.85 & 0.837 & 0.837 & 0.554 & 0.947 & 0.884 & 1.004 & 0.844 & 0.973 & 0.998 & 0.234 & 0.985 & 0.873 & 0.842  \\

\midrule
\midrule
\multicolumn{2}{c}{qwen2.5-3B-Instruct}  & 21.4 & 35.52 & 37.33 & 29.3 & 24.44 & 17.64 & 32.56 & 21.9 & 23.29 & 67.5 & 85.37 & 43.85 & 47.5 & 51.57 & 47.75 & 39.128 \\
 
\midrule 

\multirow{3}{*}{\ourmethod{}}
 &  512 & 0.68 & 0.833 & 0.937 & 0.834 & 0.811 & 0.578 & 0.948 & 0.903 & 1.002 & 0.889 & 0.915 & 0.977 & 0.751 & 0.974 & 1.006 & 0.869  \\ 
 &  1024 & 0.713 & 0.992 & 1.022 & 1.018 & 0.841 & 0.924 & 0.984 & 0.942 & 0.99 & 0.985 & 0.983 & 0.97 & 0.858 & 0.952 & 0.995 & 0.945  \\ 
 &  2048 & 0.974 & 0.962 & 1.02 & 0.998 & 0.813 & 1.023 & 0.998 & 0.973 & 1.0 & 1.022 & 0.972 & 0.992 & 0.879 & 1.009 & 1.018 & 0.977  \\ 
 
\midrule 

\multirow{3}{*}{SampleAttention}
 &  512 & 0.392 & 0.285 & 0.741 & 0.31 & 0.478 & 0.345 & 0.932 & 0.949 & 1.092 & 0.563 & 0.828 & 0.911 & 0.135 & 0.868 & 1.068 & 0.66  \\ 
 &  1024 & 0.476 & 0.43 & 0.907 & 0.52 & 0.529 & 0.429 & 0.943 & 0.976 & 1.077 & 0.77 & 0.917 & 0.956 & 0.322 & 0.922 & 1.016 & 0.746  \\ 
 &  2048 & 0.533 & 0.523 & 0.985 & 0.619 & 0.66 & 0.687 & 0.942 & 1.014 & 1.076 & 0.867 & 0.969 & 0.971 & 0.565 & 0.966 & 1.039 & 0.828  \\ 
 
\midrule 

\multirow{3}{*}{TidalDecode}
 &  512 & 0.172 & 0.188 & 0.422 & 0.179 & 0.382 & 0.159 & 0.569 & 0.726 & 0.967 & 0.526 & 0.307 & 0.642 & 0.057 & 0.846 & 0.766 & 0.461  \\ 
 &  1024 & 0.245 & 0.273 & 0.654 & 0.249 & 0.426 & 0.162 & 0.67 & 0.794 & 1.035 & 0.681 & 0.456 & 0.775 & 0.107 & 0.917 & 0.901 & 0.556  \\ 
 &  2048 & 0.34 & 0.418 & 0.691 & 0.368 & 0.564 & 0.234 & 0.814 & 0.864 & 1.068 & 0.904 & 0.693 & 0.882 & 0.094 & 0.981 & 0.967 & 0.659  \\ 
 
\midrule 

\multirow{3}{*}{SparQ}
 &  512 & 0.381 & 0.262 & 0.708 & 0.324 & 0.461 & 0.266 & 0.911 & 0.922 & 1.075 & 0.489 & 0.795 & 0.892 & 0.151 & 0.881 & 1.022 & 0.636  \\ 
 &  1024 & 0.423 & 0.446 & 0.85 & 0.446 & 0.529 & 0.363 & 0.932 & 0.966 & 1.078 & 0.696 & 0.894 & 0.954 & 0.339 & 0.922 & 1.046 & 0.726  \\ 
 &  2048 & 0.421 & 0.562 & 0.939 & 0.622 & 0.573 & 0.582 & 0.945 & 1.008 & 1.075 & 0.867 & 0.947 & 0.957 & 0.611 & 0.965 & 1.044 & 0.808  \\ 
 
\midrule 

\multirow{3}{*}{Loki}
 &  512 & 0.34 & 0.268 & 0.612 & 0.282 & 0.437 & 0.208 & 0.857 & 0.883 & 1.097 & 0.489 & 0.604 & 0.796 & 0.092 & 0.892 & 0.973 & 0.589  \\ 
 &  1024 & 0.438 & 0.408 & 0.773 & 0.289 & 0.563 & 0.304 & 0.884 & 0.937 & 1.081 & 0.674 & 0.639 & 0.908 & 0.23 & 0.962 & 0.982 & 0.671  \\ 
 &  2048 & 0.494 & 0.538 & 0.95 & 0.52 & 0.645 & 0.607 & 0.956 & 1.006 & 1.074 & 0.822 & 0.884 & 0.932 & 0.432 & 0.978 & 0.961 & 0.787  \\ 

 \midrule
\midrule
\multicolumn{2}{c}{smollm3}  & 19.46 & 37.83 & 43.11 & 18.42 & 19.58 & 10.02 & 34.18 & 23.1 & 26.78 & 76.0 & 84.86 & 44.94 & 80.96 & 67.19 & 63.1 & 43.302 \\ 
 
\midrule 

\multirow{3}{*}{\ourmethod{}}
 &  512 & 0.89 & 0.919 & 1.003 & 1.243 & 1.201 & 0.915 & 0.942 & 0.924 & 0.991 & 0.921 & 1.031 & 0.974 & 1.017 & 1.008 & 0.987 & 0.998  \\ 
 &  1024 & 0.955 & 0.985 & 1.007 & 1.313 & 1.162 & 1.109 & 0.981 & 0.977 & 0.981 & 0.947 & 1.02 & 0.981 & 1.032 & 1.01 & 0.996 & 1.03  \\ 
 &  2048 & 0.917 & 1.024 & 1.012 & 1.244 & 1.135 & 1.135 & 0.988 & 0.99 & 0.996 & 0.974 & 1.014 & 1.006 & 0.989 & 1.006 & 0.996 & 1.028  \\ 
 
\midrule 

\multirow{3}{*}{SampleAttention}
 &  512 & 0.714 & 0.662 & 0.879 & 0.926 & 0.876 & 0.738 & 0.919 & 0.928 & 0.984 & 0.809 & 1.025 & 0.959 & 0.457 & 1.026 & 0.939 & 0.856  \\ 
 &  1024 & 0.748 & 0.837 & 0.918 & 1.088 & 1.008 & 0.856 & 0.958 & 0.941 & 0.984 & 0.908 & 1.032 & 0.986 & 0.685 & 1.003 & 0.984 & 0.929  \\ 
 &  2048 & 0.797 & 0.987 & 0.965 & 1.062 & 1.01 & 0.953 & 0.984 & 0.943 & 0.993 & 0.954 & 1.035 & 0.975 & 0.848 & 1.01 & 0.976 & 0.966  \\ 
 
\midrule 

\multirow{3}{*}{TidalDecode}
 &  512 & 0.684 & 0.615 & 0.715 & 0.72 & 0.689 & 0.637 & 0.815 & 0.883 & 0.96 & 0.842 & 0.887 & 0.951 & 0.21 & 0.967 & 0.902 & 0.765  \\ 
 &  1024 & 0.761 & 0.778 & 0.808 & 0.742 & 0.812 & 0.658 & 0.893 & 0.9 & 0.987 & 0.888 & 0.976 & 0.965 & 0.544 & 0.976 & 0.937 & 0.842  \\ 
 &  2048 & 0.831 & 0.936 & 0.903 & 0.844 & 0.868 & 0.738 & 0.941 & 0.939 & 0.993 & 0.961 & 1.007 & 0.962 & 0.884 & 0.981 & 0.982 & 0.918  \\ 
 
\midrule 

\multirow{3}{*}{SparQ}
 &  512 & 0.391 & 0.58 & 0.836 & 0.656 & 0.771 & 0.535 & 0.898 & 0.881 & 1.006 & 0.704 & 0.886 & 0.875 & 0.099 & 0.973 & 0.785 & 0.725  \\ 
 &  1024 & 0.598 & 0.865 & 0.886 & 0.699 & 0.886 & 0.52 & 0.953 & 0.92 & 0.987 & 0.829 & 0.951 & 0.906 & 0.26 & 0.978 & 0.841 & 0.805  \\ 
 &  2048 & 0.634 & 0.985 & 0.967 & 1.003 & 0.975 & 0.737 & 0.997 & 0.942 & 0.993 & 0.941 & 0.995 & 0.966 & 0.78 & 1.001 & 0.921 & 0.922  \\ 
 
\midrule 

\multirow{3}{*}{Loki}
 &  512 & 0.063 & 0.195 & 0.423 & 0.143 & 0.423 & 0.111 & 0.246 & 0.437 & 0.903 & 0.467 & 0.318 & 0.488 & 0.022 & 0.887 & 0.626 & 0.384  \\ 
 &  1024 & 0.542 & 0.819 & 0.86 & 0.791 & 0.913 & 0.637 & 0.787 & 0.913 & 0.991 & 0.796 & 0.97 & 0.93 & 0.156 & 0.984 & 0.933 & 0.801  \\ 
 &  2048 & 0.206 & 0.63 & 0.73 & 0.305 & 0.871 & 0.231 & 0.599 & 0.787 & 0.99 & 0.803 & 0.556 & 0.79 & 0.043 & 0.985 & 0.804 & 0.622  \\ 

\midrule
\midrule
\multicolumn{2}{c}{qwen3-4B}  & 28.02 & 43.75 & 53.46 & 55.55 & 43.87 & 31.82 & 32.48 & 24.81 & 25.08 & 73.5 & 88.26 & 43.69 & 96.5 & 64.14 & 59.02 & 50.93 \\ 
 
\midrule 

\multirow{3}{*}{\ourmethod{}}
 &  512 & 0.814 & 0.964 & 0.996 & 0.974 & 0.89 & 1.054 & 0.996 & 0.94 & 1.032 & 0.946 & 1.008 & 1.005 & 0.876 & 0.985 & 1.009 & 0.966  \\ 
 &  1024 & 0.914 & 0.986 & 1.033 & 0.988 & 1.016 & 0.963 & 0.992 & 0.963 & 1.02 & 0.993 & 1.024 & 1.028 & 0.959 & 0.987 & 1.012 & 0.992  \\ 
 &  2048 & 0.876 & 0.992 & 1.024 & 1.015 & 0.977 & 0.995 & 0.989 & 0.988 & 1.008 & 0.993 & 1.011 & 1.023 & 1.026 & 1.008 & 1.004 & 0.995  \\ 
 
\midrule 

\multirow{3}{*}{SampleAttention}
 &  512 & 0.525 & 0.483 & 0.709 & 0.725 & 0.697 & 0.536 & 0.966 & 0.822 & 1.025 & 0.871 & 0.963 & 0.957 & 0.197 & 0.949 & 0.894 & 0.755  \\ 
 &  1024 & 0.665 & 0.731 & 0.834 & 0.869 & 0.803 & 0.627 & 0.991 & 0.876 & 1.02 & 0.966 & 0.971 & 1.003 & 0.772 & 1.006 & 0.984 & 0.875  \\ 
 &  2048 & 0.801 & 0.947 & 0.972 & 0.892 & 0.9 & 0.72 & 0.994 & 0.923 & 0.998 & 1.0 & 0.989 & 1.013 & 1.021 & 1.0 & 1.03 & 0.947  \\ 
 
\midrule 

\multirow{3}{*}{TidalDecode}
 &  512 & 0.401 & 0.513 & 0.597 & 0.583 & 0.582 & 0.395 & 0.774 & 0.842 & 0.96 & 0.796 & 0.76 & 0.915 & 0.047 & 0.936 & 0.878 & 0.665  \\ 
 &  1024 & 0.483 & 0.682 & 0.731 & 0.709 & 0.728 & 0.474 & 0.866 & 0.879 & 0.996 & 0.871 & 0.917 & 0.966 & 0.358 & 0.996 & 0.931 & 0.773  \\ 
 &  2048 & 0.537 & 0.83 & 0.831 & 0.903 & 0.781 & 0.586 & 0.925 & 0.919 & 1.01 & 0.952 & 0.968 & 0.982 & 0.803 & 0.996 & 1.002 & 0.868  \\ 
 
\midrule 

\multirow{3}{*}{SparQ}
 &  512 & 0.459 & 0.525 & 0.606 & 0.544 & 0.702 & 0.448 & 0.906 & 0.832 & 1.039 & 0.782 & 0.824 & 0.914 & 0.048 & 0.878 & 0.789 & 0.686  \\ 
 &  1024 & 0.516 & 0.749 & 0.744 & 0.722 & 0.792 & 0.567 & 0.948 & 0.867 & 1.032 & 0.891 & 0.952 & 0.963 & 0.14 & 0.964 & 0.878 & 0.782  \\ 
 &  2048 & 0.527 & 0.895 & 0.923 & 0.795 & 0.739 & 0.797 & 0.985 & 0.917 & 1.012 & 0.966 & 0.975 & 1.002 & 0.585 & 0.993 & 0.962 & 0.872  \\ 
 
\midrule 

\multirow{3}{*}{Loki}
 &  512 & 0.295 & 0.328 & 0.687 & 0.554 & 0.585 & 0.426 & 0.934 & 0.804 & 1.056 & 0.884 & 0.756 & 0.594 & 0.076 & 0.715 & 0.637 & 0.622  \\ 
 &  1024 & 0.416 & 0.605 & 0.822 & 0.694 & 0.71 & 0.563 & 0.986 & 0.864 & 1.031 & 0.912 & 0.869 & 0.721 & 0.171 & 0.826 & 0.73 & 0.728  \\ 
 &  2048 & 0.543 & 0.822 & 0.898 & 0.837 & 0.824 & 0.781 & 1.006 & 0.902 & 1.016 & 0.971 & 0.946 & 0.892 & 0.404 & 0.94 & 0.824 & 0.841  \\

\bottomrule
\end{tabular}
}
\end{center}
\end{table*}

%% file: tables/longbench_ggattn_appendix.tex
\begin{table*}[!t]
\caption{\textbf{LongBench results (Higher is better)}. Each model utilizes chunked prefill with $B=128$ with different selective budgets. Raw scores reported for base models and relative errors (percent of performance compared to baseline) presented for \ourmethod{}.}
\label{table:longbench_compare_appendix2}
\vspace{-5mm}
\begin{center}
\resizebox{\textwidth}{!}{
\begin{tabular}{ccccccccccccccccccc}
\toprule

%%%%%%%%%%%%%%%%%%%%%%%%%%%%%%%%%%%%%%%%%%%%%%%%%%%%%%%%
%%%%%%%%%%%%%%%%%%%%% Header Rows %%%%%%%%%%%%%%%%%%%%%&
%%%%%%%%%%%%%%%%%%%%%%%%%%%%%%%%%%%%%%%%%%%%%%%%%%%%%%%%
\multirow{2}{*}{Method}
&
\multirow{2}{*}{Budget}
&
\multicolumn{3}{c}{Single Doc. QA} &
\multicolumn{3}{c}{Multi Doc. QA} &
\multicolumn{3}{c}{Summarization} &
\multicolumn{3}{c}{Few$\-$shot Learning} &
\multicolumn{1}{c}{Synthetic} &
\multicolumn{2}{c}{Code} &
\\

\cmidrule(lr){3-5}
\cmidrule(lr){6-8}
\cmidrule(lr){9-11}
\cmidrule(lr){12-14}
\cmidrule(lr){15-15}
\cmidrule(lr){16-17}

&
& 
{Narrative QA} & {Qasper} & {MF-en} & 
{HotpotQA} & {2WikiMQA} & {Musique} & 
{GovReport} & {QMSum} & {MultiNews} & 
{TREC} & {TriviaQA} & {SAMSum} &  {PR-en} &
{Lcc} & {RB-P} & 
{Avg.} 
\\

\midrule
\midrule
\multicolumn{2}{c}{llama3.1-8B-Instruct}  & 31.16 & 46.94 & 56.45 & 58.03 & 48.24 & 31.97 & 34.79 & 25.55 & 26.94 & 73.0 & 91.72 & 43.2 & 99.5 & 62.06 & 52.55 & 52.14 \\ 
\midrule 

\multirow{4}{*}{\ourmethod{}}
 &  256 & 0.713 & 0.909 & 0.89 & 0.823 & 0.97 & 0.874 & 0.958 & 0.895 & 1.001 & 0.836 & 0.997 & 1.027 & 0.879 & 1.079 & 1.021 & 0.925 \\
 &  512 & 0.873 & 0.954 & 0.925 & 0.903 & 0.971 & 0.831 & 0.966 & 0.926 & 1.008 & 0.938 & 1.002 & 1.033 & 0.985 & 1.055 & 1.056 & 0.962 \\
 &  1024 & 0.889 & 0.963 & 0.935 & 0.95 & 1.021 & 0.884 & 0.993 & 0.964 & 1.001 & 0.966 & 1.003 & 1.021 & 0.995 & 1.034 & 1.035 & 0.977 \\
 &  2048 & 0.992 & 0.991 & 0.948 & 0.965 & 1.036 & 0.917 & 0.998 & 0.999 & 1.002 & 0.979 & 1.004 & 1.014 & 1.0 & 1.01 & 1.022 & 0.992 \\
 &  4096 & 1.003 & 0.988 & 0.998 & 0.962 & 1.017 & 0.911 & 0.997 & 1.0 & 1.001 & 0.986 & 1.008 & 1.009 & 1.0 & 1.009 & 1.005 & 0.993 \\

\midrule
\midrule
\multicolumn{2}{c}{llama3.2-3B-Instruct}  & 22.91 & 40.49 & 49.99 & 50.96 & 43.29 & 26.97 & 33.42 & 24.28 & 24.98 & 73.5 & 90.17 & 42.04 & 96.0 & 56.53 & 56.95 & 48.832 \\ 
\midrule 

\multirow{4}{*}{\ourmethod{}}
 &  256 & 0.773 & 0.959 & 0.907 & 0.86 & 0.825 & 0.677 & 0.931 & 0.854 & 1.034 & 0.837 & 0.982 & 0.979 & 0.396 & 1.031 & 0.952 & 0.866 \\
 &  512 & 0.844 & 0.98 & 0.979 & 0.978 & 0.964 & 0.765 & 0.952 & 0.864 & 1.042 & 0.973 & 1.0 & 0.995 & 0.766 & 1.074 & 1.003 & 0.945 \\
 &  1024 & 0.867 & 1.001 & 0.995 & 1.005 & 1.005 & 0.786 & 0.978 & 0.908 & 1.034 & 1.007 & 1.005 & 1.003 & 0.927 & 1.034 & 1.021 & 0.972 \\
 &  2048 & 0.973 & 0.991 & 1.029 & 1.019 & 0.982 & 0.863 & 0.984 & 0.96 & 1.004 & 1.007 & 1.006 & 0.987 & 0.969 & 1.016 & 0.996 & 0.986 \\
 &  4096 & 0.955 & 0.996 & 1.019 & 1.024 & 1.003 & 0.959 & 0.99 & 0.968 & 0.995 & 1.0 & 1.007 & 1.002 & 0.995 & 0.999 & 0.997 & 0.994 \\

\midrule
\midrule
\multicolumn{2}{c}{qwen2.5-7B-Instruct}  & 25.84 & 37.79 & 43.02 & 47.66 & 40.11 & 25.22 & 33.52 & 22.49 & 23.12 & 66.5 & 87.46 & 44.71 & 96.25 & 57.76 & 61.77 & 47.548 \\ 
\midrule 

\multirow{4}{*}{\ourmethod{}}
 &  256 & 0.594 & 0.588 & 0.736 & 0.616 & 0.565 & 0.433 & 0.933 & 0.835 & 0.959 & 0.812 & 0.959 & 0.975 & 0.299 & 0.854 & 0.692 & 0.723 \\
 &  512 & 0.811 & 0.904 & 0.98 & 0.85 & 0.798 & 0.721 & 0.958 & 0.884 & 0.98 & 0.932 & 0.998 & 1.0 & 0.777 & 0.966 & 0.865 & 0.895 \\
 &  1024 & 0.803 & 0.983 & 0.97 & 0.996 & 0.84 & 0.856 & 0.97 & 0.954 & 0.996 & 0.925 & 0.977 & 1.02 & 0.921 & 0.99 & 0.94 & 0.943 \\
 &  2048 & 0.894 & 0.927 & 0.97 & 1.001 & 0.856 & 0.851 & 0.987 & 0.98 & 0.999 & 0.947 & 0.989 & 1.029 & 0.949 & 1.007 & 0.967 & 0.957 \\
 &  4096 & 0.937 & 0.982 & 1.032 & 1.0 & 0.931 & 0.905 & 0.987 & 1.004 & 0.998 & 1.0 & 0.992 & 1.014 & 0.968 & 1.001 & 0.99 & 0.983 \\

\midrule
\midrule
\multicolumn{2}{c}{qwen2.5-3B-Instruct}  & 21.4 & 35.52 & 37.33 & 29.3 & 24.44 & 17.64 & 32.56 & 21.9 & 23.29 & 67.5 & 85.37 & 43.85 & 47.5 & 51.57 & 47.75 & 39.128 \\ 
\midrule 

\multirow{4}{*}{\ourmethod{}}
 &  256 & 0.495 & 0.548 & 0.711 & 0.459 & 0.705 & 0.337 & 0.891 & 0.844 & 0.991 & 0.644 & 0.782 & 0.937 & 0.563 & 0.92 & 0.972 & 0.72 \\
 &  512 & 0.68 & 0.833 & 0.937 & 0.834 & 0.811 & 0.578 & 0.948 & 0.903 & 1.002 & 0.889 & 0.915 & 0.977 & 0.751 & 0.974 & 1.006 & 0.869 \\
 &  1024 & 0.713 & 0.992 & 1.022 & 1.018 & 0.841 & 0.924 & 0.984 & 0.942 & 0.99 & 0.985 & 0.983 & 0.97 & 0.858 & 0.952 & 0.995 & 0.945 \\
 &  2048 & 0.974 & 0.962 & 1.02 & 0.998 & 0.813 & 1.023 & 0.998 & 0.973 & 1.0 & 1.022 & 0.972 & 0.992 & 0.879 & 1.009 & 1.018 & 0.977 \\
 &  4096 & 0.943 & 1.04 & 1.034 & 0.968 & 0.913 & 1.022 & 0.998 & 0.997 & 1.001 & 0.985 & 0.994 & 0.994 & 0.911 & 0.985 & 0.99 & 0.985 \\

\midrule
\midrule
\multicolumn{2}{c}{qwen3-8B}  & 26.44 & 47.73 & 53.7 & 59.34 & 43.45 & 34.77 & 33.33 & 24.13 & 24.94 & 71.5 & 90.71 & 44.33 & 100.0 & 69.01 & 62.0 & 52.359 \\ 
\midrule 

\multirow{4}{*}{\ourmethod{}}
 &  256 & 0.696 & 0.819 & 0.859 & 0.763 & 0.795 & 0.572 & 0.993 & 0.848 & 0.986 & 0.93 & 0.978 & 0.985 & 0.885 & 0.973 & 0.895 & 0.865 \\
 &  512 & 0.89 & 0.921 & 0.987 & 1.007 & 0.938 & 0.868 & 1.005 & 0.951 & 0.998 & 1.007 & 0.996 & 1.01 & 1.0 & 1.009 & 1.03 & 0.974 \\
 &  1024 & 0.928 & 0.999 & 0.998 & 0.987 & 0.954 & 0.903 & 1.002 & 0.981 & 1.001 & 1.007 & 1.006 & 1.023 & 1.0 & 1.007 & 1.034 & 0.988 \\
 &  2048 & 0.958 & 0.992 & 1.014 & 1.024 & 0.932 & 0.897 & 1.007 & 0.99 & 1.001 & 1.007 & 0.999 & 1.014 & 1.0 & 0.999 & 1.008 & 0.989 \\
 &  4096 & 0.965 & 0.996 & 0.999 & 1.024 & 0.962 & 0.985 & 1.003 & 0.998 & 1.0 & 1.0 & 1.0 & 0.998 & 1.0 & 0.996 & 1.003 & 0.995 \\

\midrule
\midrule
\multicolumn{2}{c}{qwen3-4B}  & 28.02 & 43.75 & 53.46 & 55.55 & 43.87 & 31.82 & 32.48 & 24.81 & 25.08 & 73.5 & 88.26 & 43.69 & 96.5 & 64.14 & 59.02 & 50.93 \\ 
\midrule 

\multirow{4}{*}{\ourmethod{}}
 &  256 & 0.627 & 0.884 & 0.918 & 0.789 & 0.804 & 0.717 & 0.99 & 0.88 & 1.02 & 0.789 & 0.98 & 0.947 & 0.741 & 0.937 & 0.925 & 0.863 \\
 &  512 & 0.814 & 0.964 & 0.996 & 0.974 & 0.89 & 1.054 & 0.996 & 0.94 & 1.032 & 0.946 & 1.008 & 1.005 & 0.876 & 0.985 & 1.009 & 0.966 \\
 &  1024 & 0.914 & 0.986 & 1.033 & 0.988 & 1.016 & 0.963 & 0.992 & 0.963 & 1.02 & 0.993 & 1.024 & 1.028 & 0.959 & 0.987 & 1.012 & 0.992 \\
 &  2048 & 0.876 & 0.992 & 1.024 & 1.015 & 0.977 & 0.995 & 0.989 & 0.988 & 1.008 & 0.993 & 1.011 & 1.023 & 1.026 & 1.008 & 1.004 & 0.995 \\
 &  4096 & 0.91 & 1.001 & 1.011 & 1.009 & 0.997 & 0.989 & 0.994 & 0.999 & 1.002 & 1.007 & 1.001 & 0.991 & 1.021 & 1.0 & 1.003 & 0.996 \\

\midrule
\midrule
\multicolumn{2}{c}{qwen3-1.7B}  & 18.94 & 25.17 & 46.48 & 39.07 & 32.52 & 18.06 & 30.73 & 22.88 & 24.77 & 73.5 & 85.39 & 42.37 & 94.0 & 44.6 & 37.81 & 42.419 \\ 
\midrule 

\multirow{4}{*}{\ourmethod{}}
 &  256 & 0.801 & 0.947 & 0.843 & 0.711 & 0.8 & 0.587 & 0.988 & 0.91 & 1.007 & 0.864 & 0.933 & 0.959 & 0.646 & 0.832 & 0.726 & 0.837 \\
 &  512 & 0.987 & 0.986 & 0.948 & 0.88 & 0.984 & 0.822 & 1.032 & 0.97 & 1.027 & 0.946 & 1.01 & 0.978 & 0.949 & 0.933 & 0.841 & 0.953 \\
 &  1024 & 1.041 & 1.003 & 0.983 & 0.973 & 1.052 & 0.865 & 1.021 & 0.996 & 1.01 & 0.98 & 1.007 & 0.999 & 0.979 & 1.017 & 0.834 & 0.984 \\
 &  2048 & 0.968 & 0.995 & 0.967 & 0.976 & 1.015 & 1.007 & 1.005 & 1.003 & 1.0 & 0.98 & 1.013 & 1.0 & 1.0 & 1.001 & 0.869 & 0.986 \\
 &  4096 & 0.982 & 1.019 & 0.987 & 0.957 & 1.026 & 1.048 & 0.988 & 1.025 & 0.997 & 1.0 & 1.006 & 1.014 & 0.995 & 1.0 & 0.973 & 1.001 \\
\midrule
\midrule
\multicolumn{2}{c}{qwen3-30B-A3B-2507}  & 31.33 & 42.19 & 55.0 & 63.0 & 55.94 & 33.85 & 31.06 & 21.79 & 23.58 & 76.72 & 90.42 & 46.41 & 100.0 & 74.33 & 67.08 & 54.18 \\ 
 
\midrule 

\multirow{4}{*}{\ourmethod{}}
 &  256 & 0.658 & 1.048 & 0.887 & 0.862 & 0.604 & 0.61 & 1.011 & 0.904 & 0.978 & 0.899 & 0.987 & 1.0 & 1.0 & 0.979 & 0.859 & 0.886  \\ 
 &  512 & 0.903 & 1.022 & 0.95 & 0.988 & 0.897 & 0.964 & 1.006 & 0.976 & 0.993 & 0.971 & 1.002 & 1.003 & 1.0 & 1.007 & 0.992 & 0.978  \\ 
 &  1024 & 0.991 & 1.015 & 0.953 & 1.005 & 0.956 & 0.985 & 1.0 & 1.007 & 1.003 & 0.997 & 1.015 & 1.009 & 1.0 & 1.002 & 1.005 & 0.996  \\ 
 &  2048 & 0.992 & 1.021 & 0.98 & 1.005 & 0.974 & 1.012 & 1.009 & 1.002 & 0.998 & 0.997 & 1.015 & 1.012 & 1.0 & 0.996 & 1.009 & 1.001  \\ 
 &  4096 & 1.011 & 1.007 & 0.983 & 1.003 & 1.005 & 0.974 & 1.002 & 0.997 & 0.996 & 0.997 & 1.007 & 0.994 & 1.0 & 0.999 & 1.004 & 0.999  \\

\midrule
\midrule
\multicolumn{2}{c}{smollm3}  & 19.46 & 37.83 & 43.11 & 18.42 & 19.58 & 10.02 & 34.18 & 23.1 & 26.78 & 76.0 & 84.86 & 44.94 & 80.96 & 67.19 & 63.1 & 43.302 \\ 
\midrule 

\multirow{4}{*}{\ourmethod{}}
 &  256 & 0.777 & 0.847 & 0.919 & 1.127 & 1.097 & 0.848 & 0.937 & 0.928 & 0.979 & 0.757 & 1.009 & 0.939 & 0.704 & 0.965 & 0.894 & 0.915 \\
 &  512 & 0.89 & 0.919 & 1.003 & 1.243 & 1.201 & 0.915 & 0.942 & 0.924 & 0.991 & 0.921 & 1.031 & 0.974 & 1.017 & 1.008 & 0.987 & 0.998 \\
 &  1024 & 0.955 & 0.985 & 1.007 & 1.313 & 1.162 & 1.109 & 0.981 & 0.977 & 0.981 & 0.947 & 1.02 & 0.981 & 1.032 & 1.01 & 0.996 & 1.03 \\
 &  2048 & 0.917 & 1.024 & 1.012 & 1.244 & 1.135 & 1.135 & 0.988 & 0.99 & 0.996 & 0.974 & 1.014 & 1.006 & 0.989 & 1.006 & 0.996 & 1.028 \\
 &  4096 & 0.891 & 1.002 & 1.008 & 1.226 & 1.046 & 1.004 & 1.016 & 0.997 & 0.995 & 1.0 & 0.993 & 0.999 & 1.023 & 0.999 & 0.996 & 1.013 \\

\bottomrule
\end{tabular}
}
\vspace{-5mm}
\end{center}
\end{table*}

%% file: contents/appendix_math_500.tex
\section{Math 500 Results}

The MATH-500 benchmark is a test designed to challenge the mathematical reasoning capabilities of LLMs. It comprises 500 problems sourced from high-level math competitions like the AMC and AIME, spanning domains such as algebra, geometry, number theory, precalculus, combinatorics, and probability. Unlike simpler datasets, MATH-500 emphasizes multi-step problem solving and abstract reasoning, requiring models to produce precise, step-by-step solutions. It has become a key benchmark for comparing LLMs' mathematical proficiency and due to long reasoning traces, is a good way to test the effectiveness of selective attention models during generation. While the majority of focus on \ourmethod{} has been on its effectiveness with chunked prefill, it also performs remarkably well during decode as well. This is seen in \cref{table:math500} in which \ourmethod{} performs as well or better than other competing methods, with much smaller reasoning traces generally resulting in significantly shorter reasoning times. This suggests that more information is collated during reasoning with \ourmethod{} than competing methods.

\input{tables/math500}

%% file: tables/math500.tex
\begin{table*}[htbp!]
\caption{\textbf{Math 500 results} across different selective budgets utilizing \ourmethod{}. Reasoning generation length limited to 8192 tokens.}
\label{table:math500}
\begin{center}
\resizebox{.5\textwidth}{!}{
\begin{tabular}{cccccc}
\toprule

& 
{Budget} &
{Flex Match} & {Exact Match} & 
{Avg. Gen. Length}
\\

\midrule
\multicolumn{2}{c}{\hspace{-15mm} GPT-OSS-20b} &   0.893  &  0.694  &  2558.262 \\ 
\midrule
\midrule
\multirow{2}{*}{SparQ}
  & 128 & 0.744 &  0.569  &  3023.71 \\
& 256 &  0.848  &  0.648  &  2436.86 \\
\midrule

\multirow{2}{*}{Loki}
  & 128&  0.905  &  {\bf{0.711}}  & 2276.0 \\
& 256 &  0.903 &  {\bf{0.714}}  & 2181.49 \\
\midrule
\multirow{2}{*}{LessIsMore} 
   & 128 &  0.782  &  0.62  &  3591.4 \\
 & 256 &  0.85  &  0.681  &  2830.49  \\
\midrule

\multirow{2}{*}{\ourmethod{}}
&  128 &  {\bf{0.911}}  &  0.707  &  2199.93 \\
&  256 &  {\bf{0.913}} &  {0.711}  &  2539.11 \\

\midrule
\multicolumn{2}{c}{\hspace{-15mm} Qwen3-4B} &   0.82  &  0.701  &  1179.78 \\ 

\midrule
\midrule
\multirow{2}{*}{SparQ}
  & 128 & 0.727 &  0.615  &  2044.902 \\
& 256 &  0.777  &  0.658  &  1518.82 \\
\midrule

\multirow{2}{*}{Loki}
  & 128 &  0.785  &  0.668  & 1345.88 \\
& 256 &  0.821 &  0.702  & 1143.51 \\
\midrule
\multirow{2}{*}{LessIsMore} 
   & 128 &  0.729  & 0.621  &  2674.25 \\
 & 256 &  0.814  &  0.695  &  1618.00  \\
\midrule

\multirow{2}{*}{\ourmethod{}}
&  128 &  {\bf{0.828}}  &  {\bf{0.708}}  &  1119.46 \\
&  256 &  {\bf{0.823}} &  {\bf{0.703}}  &  1156.12 \\
\midrule

\multicolumn{2}{c}{\hspace{-15mm} Smollm3} &   0.659  &  0.522  &  629.04 \\ 

\midrule
\midrule
\multirow{2}{*}{SparQ}
  & 128 & 0.647 &  0.514  &  715.18\\
& 256 &  {\bf{0.665}}  &  {\bf{0.532}}  &  634.34 \\
\midrule

\multirow{2}{*}{Loki}
  & 128&  0.615  &  0.494  & 621.02 \\
& 256 &  0.654 &  0.52  & 618.47 \\
\midrule

\multirow{2}{*}{LessIsMore} 
   & 128 &  0.623  &  0.495  & 642.26  \\
 & 256 &  0.657  &  0.527  &  598.907  \\
\midrule

\multirow{2}{*}{\ourmethod{}}
&  128 &  \bf{0.66}  &  \bf{0.521}  &  578.10 \\
&  256 &  0.661 &  0.527  &  614.73 \\

\bottomrule
\end{tabular}
}

\end{center}
\end{table*}

%% file: contents/appendix_ablations.tex
\section{Ablations}
\input{tables/ablation_score_aggr}

In this section we report the results of several ablation studies. In \cref{table:ruler_ablation_cossim} we explore the effects of using cosine similarity instead of dot products to estimate the attention scores prior to aggregation in \ourmethod{}. Specifically, we apply our method varying both scoring methods in the Llama3.2-3B-Instruct model on the RULER benchmark, where it is clear that the cosine similarity provides more normalized, better aggreated scores for selective attention. In \cref{table:ruler_ablation} we demonstrate the differences between aggregation across the query dimension of the approximated attention scores using both the max and the mean. Again, the Llama3.2-3B-Instruct model and RULER benchmark were used, and the higher scores indicated that the max captures important outlying key query interactions better than the mean.

In order to determine if \ourmethod{} can be reliably used with different prefill chunk sizes, we use LongBench to test the Qwen3-4B model with both \ourmethod{} and sample attention across different $B_\text{CP}$. The results are presented in \cref{table:longbench_ablation_bpp} in which the number of queries selected $N_Q$ is $25 \%$ of $B_\text{CP}$ and the selctive budget $B_\text{SA} = 1024$. In this case we see no performance degradation when varying the prefill chunk size suggesting \ourmethod{} is robust to changes in this parameter.

\input{tables/longbench_ablation_bpp}

To explore the effects of varying the number of subselected queries during chunked prefill, we compared the performance of the Qwen3-4B model on the LongBench task using both \ourmethod{} and SampleAttention. During prefill, the number of queries subselected varied between $4$ and $128$. These results are presented in \cref{table:longbench_ablation_nq} in which it is clear that very little performance is lost even with a small number of queries used to approximate attention scores for KV subselection. This performance, particularly compared to the closest competitor SampleAttention, seems to justify our query subselection method.

\input{tables/longbench_ablation_nq}

%% file: tables/ablation_score_aggr.tex
\begin{minipage}[t]{0.45\textwidth}
\centering
\captionsetup{hypcap=false}
\captionof{table}{Scoring Ablation}
\begin{tabular}{ccccc}
\toprule

\multirow{2.5}{*}{Scoring} & \multicolumn{4}{c}{Ruler Test Length} \\
\cmidrule(lr){2-5} 
&
{4096} & {8192} & 
{16384} & {32768}
\\

\midrule
Dot Prod.  & 83.76 & 68.22 & 57.97& 46.24\\
Cos. Sim & \bf{88.56} & \bf{73.78} & \bf{65.57} & \bf{52.48}\\
\hline
\end{tabular}

\label{table:ruler_ablation_cossim}
\end{minipage}
\hfill
\begin{minipage}[t]{0.45\textwidth}
\centering
\captionsetup{hypcap=false}
\captionof{table}{Aggregation Ablation}
\begin{tabular}{ccccc}
\toprule

\multirow{2.5}{*}{Aggr.} & \multicolumn{4}{c}{Ruler Test Length} \\
\cmidrule(lr){2-5} 
&
{4096} & {8192} & 
{16384} & {32768}
\\

\midrule
 Mean & 84.28 & 69.14 & 58.81& 45.53\\
 Max & \bf{88.56} & \bf{73.78} & \bf{65.57} & \bf{52.48}\\
\hline
\end{tabular}
\label{table:ruler_ablation}
\end{minipage}

%% file: tables/longbench_ablation_bpp.tex
% \begin{table}[!t]
% \caption{\textbf{LongBench results for the Qwen3-4B model utilizing \ourmethod{} across $B_\text{CP}\in\{128, 256, 512\}$ with $B_\text{SA}=1024$ (Higher is better)}. Set $N_q$ to 25\% of $B_\text{CP}$. Relative scores compared to the dense model are reported.}
% \label{table:longbench_ablation_bpp}
% \begin{center}
% \resizebox{0.6\textwidth}{!}{
% \begin{tabular}{c ccc ccc}
% \toprule
% & \multicolumn{3}{c}{\ourmethod{}} & \multicolumn{3}{c}{SampleAttention} \\
% \cmidrule(lr){2-4} \cmidrule(lr){5-7}
% $B_\text{CP}$ & 128 & 256 & 512 & 128 & 256 & 512 \\
% \midrule
% \textit{Qwen3-4B} & \bf{0.978} & \bf{0.982} & \bf{0.977} & 0.876 & 0.871 & 0.872 \\
% \bottomrule
% \end{tabular}
% }
% \end{center}
% \end{table}

\begin{table}[!t]
\caption{\textbf{LongBench results for the Qwen3-4B model utilizing \ourmethod{} across $B_\text{CP}\in\{128, 256, 512\}$ with $B_\text{SA}=1024$ (Higher is better)}. Set $N_q$ to 25\% of $B_\text{CP}$. Relative scores compared to the dense model are reported.}
\label{table:longbench_ablation_bpp}
\begin{center}
\resizebox{0.4\textwidth}{!}{
\begin{tabular}{c ccc}
\toprule
% & \multicolumn{3}{c}{\ourmethod{}} & \multicolumn{3}{c}{SampleAttention} \\
% \cmidrule(lr){2-4} \cmidrule(lr){5-7}
$B_\text{CP}$ & 128 & 256 & 512 \\
\midrule
\ourmethod{} & \bf{0.978} & \bf{0.982} & \bf{0.977} \\
Sample Attention & 0.876 & 0.871 & 0.872 \\
\bottomrule
\end{tabular}
}
\end{center}
\end{table}

%% file: tables/longbench_ablation_nq.tex
% \begin{table}[!t]
% \caption{\textbf{LongBench results for the Qwen3-4B model utilizing \ourmethod{} across $N_Q\in\{4,8,16,32,64,128\}$ with $B_\text{SA}=1024$ and $B_{\text{CP}} = 128$ (Higher is better)}.  Relative scores compared to the dense model are reported.}
% \label{table:longbench_ablation_nq}
% \begin{center}
% \resizebox{0.8\textwidth}{!}{
% \begin{tabular}{c cccccc cccccc}
% \toprule
% & \multicolumn{6}{c}{\ourmethod{}} & \multicolumn{6}{c}{SampleAttention} \\
% \cmidrule(lr){2-7} \cmidrule(lr){8-13}
% $N_Q$ & 4 & 8 & 16 & 32 & 64 & 128 & 4 & 8 & 16 & 32 & 64 & 128  \\
% \midrule
% \textit{Qwen3-4B} & \bf{0.935} & \bf{0.957} & \bf{0.965} & \bf{0.978} & \bf{0.988} & \bf{0.995} & 0.854 & 0.862 & 0.88  & 0.878 & 0.885 & 0.884  \\
% \bottomrule
% \end{tabular}
% }
% \end{center}
% \end{table}

\begin{table}[!t]
\caption{\textbf{LongBench results for the Qwen3-4B model utilizing \ourmethod{} across $N_Q\in\{4,8,16,32,64,128\}$ with $B_\text{SA}=1024$ and $B_{\text{CP}} = 128$ (Higher is better)}.  Relative scores compared to the dense model are reported.}
\label{table:longbench_ablation_nq}
\begin{center}
\resizebox{0.8\textwidth}{!}{
\begin{tabular}{c cccccc}
\toprule
% & \multicolumn{6}{c}{\ourmethod{}} & \multicolumn{6}{c}{SampleAttention} \\
% \cmidrule(lr){2-7} \cmidrule(lr){8-13}
$N_Q$ & 4 & 8 & 16 & 32 & 64 & 128 \\
\midrule
\ourmethod{} & \bf{0.935} & \bf{0.957} & \bf{0.965} & \bf{0.978} & \bf{0.988} & \bf{0.995} \\
Sample Attention & 0.854 & 0.862 & 0.88  & 0.878 & 0.885 & 0.884  \\
\bottomrule
\end{tabular}
}
\end{center}
\end{table}